%% file: neurips_2026.tex
\documentclass{article}

\PassOptionsToPackage{numbers, compress}{natbib}
\usepackage[preprint]{neurips_2026}

\usepackage[utf8]{inputenc}
\usepackage[T1]{fontenc}
\usepackage{hyperref}
\usepackage{url}
\usepackage{booktabs}
\usepackage{amsfonts}
\usepackage{nicefrac}
\usepackage{microtype}
\usepackage{xcolor}
\usepackage{pifont}

\input{setup/packages}
\input{setup/math_commands}

\input{setup/defs}

\title{Adaptive Q-Chunking for Offline-to-Online Reinforcement Learning}

\author{%
  Nandiraju Gireesh$^{1,2}$ \quad
  Yuanliang Ju$^{3}$ \quad
  He Wang$^{1,2,\dagger}$ \\
  \\
  $^{1}$~Peking University \quad
  $^{2}$~Galbot \quad
  $^{3}$~University of Toronto
}

\begin{document}

\maketitle

\begin{abstract}
Offline-to-online reinforcement learning with action chunking eliminates multi-step off-policy bias and enables temporally coherent exploration, but all existing methods use a fixed chunk size across every state.
This is suboptimal: near contact events the agent needs short chunks for reactive control, while during free-space motion long chunks provide better credit assignment.
The natural solution is to train critics for several chunk sizes and select the best one at each state, but naive comparison of learned critic values systematically collapses to the shortest chunk due to discount-scale mismatch, and degrades to noise in low-value states.
We propose \hourgreen{\ourslong{} (\hourgreen{\ours{}})}, which resolves both failures by comparing the advantage of each chunk size relative to a per-horizon baseline, normalized by the discount factor.
This criterion converts biased wrong answers into unbiased near-random choices when no genuine signal exists, and becomes discriminative when a particular scale enables better planning.
We prove theoretical bounds on the advantage selector's noise immunity and on the value dominance of adaptive chunking over any fixed chunk size.
We demonstrate that \hourgreen{\ours{}} achieves state-of-the-art offline and online success rates on OGBench and Robomimic, and can be applied to enhance the performance of large-scale VLA models that predict action sequences, significantly boosting performance on RoboCasa-GR1 tasks.
\end{abstract}

\input{sections/introduction}
\input{sections/related}
\input{sections/preliminaries}
\input{sections/motivation}
\input{sections/method}
\input{sections/experiments}
\input{sections/conclusion}

\newpage
\bibliography{neurips_2026}
\bibliographystyle{plainnat}

\newpage
%%%%%%%%%%%%%%%%%%%%%%%%%%%%%%%%%%%%%%%%%%%%%%%%%%%%%%%%%%%%
\appendix
%%%%%%%%%%%%%%%%%%%%%%%%%%%%%%%%%%%%%%%%%%%%%%%%%%%%%%%%%%%%
\input{sections/limitations}
\input{sections/appendix_experiments}
\input{sections/appendix_implementation}
\input{sections/full_results}

\newpage
\input{sections/theory}

\newpage

\end{document}

%% file: setup/packages.tex
\usepackage{algorithm}
\usepackage{multirow}
\usepackage{makecell}
\usepackage[noEnd,commentColor=black]{algpseudocodex}
\usepackage[most,breakable]{tcolorbox}

\usepackage{amsmath}
\usepackage{amssymb}
\usepackage{mathtools}
\usepackage{amsthm}
\usepackage{cancel}
\usepackage{scalerel}

\usepackage[table]{xcolor}
\usepackage[capitalize,noabbrev]{cleveref}

\theoremstyle{plain}

\newtheorem{theorem}{Theorem}[section]

\newtheorem{corollary}[theorem]{Corollary}
\newtheorem{proposition}[theorem]{Proposition}
\newtheorem{definition}[theorem]{Definition}

\newtheorem{assumption}[theorem]{Assumption}

\usepackage[textsize=tiny]{todonotes}

\usepackage[utf8]{inputenc} % allow utf-8 input
\usepackage[T1]{fontenc}    % use 8-bit T1 fonts
\usepackage{url}            % simple URL typesetting
\usepackage{booktabs}       % professional-quality tables
\usepackage{times}
\usepackage{xcolor} % color s
\usepackage{graphicx}
\usepackage{blindtext}
\usepackage[labelformat=empty, position=top]{subcaption}
\usepackage[export]{adjustbox}
\usepackage{wrapfig}
\usepackage{pifont}% http://ctan.org/pkg/pifont
\definecolor{mydarkblue}{rgb}{0,0.08,0.45}
\hypersetup{ %
    colorlinks=true,
    linkcolor=mydarkblue,
    citecolor=mydarkblue,
    filecolor=mydarkblue,
    urlcolor=mydarkblue,
}

\usepackage[shortlabels]{enumitem}

\usepackage{tikz}
\usetikzlibrary{decorations.pathreplacing,calc}

%% file: setup/math_commands.tex
\usepackage{amsmath,amsfonts,bm}

\def\eqref#1{equation~\ref{#1}}
\def\1{\bm{1}}

\DeclareMathAlphabet{\mathsfit}{\encodingdefault}{\sfdefault}{m}{sl}
\SetMathAlphabet{\mathsfit}{bold}{\encodingdefault}{\sfdefault}{bx}{n}

\DeclareMathOperator*{\argmax}{arg\,max}

%% file: setup/defs.tex
\definecolor{ourcolor}{HTML}{F97B4F}       % AQC orange
\definecolor{ourlightcolor}{HTML}{FECEA8}  % light orange for table highlights
\definecolor{qccolor}{HTML}{3552A6}        % QC blue (medium blue)
\definecolor{dqccolor}{HTML}{0076BA}       % DQC blue   (for comparison)
\definecolor{darkblue}{rgb}{0,0.08,0.45}
\definecolor{cred}{HTML}{D62728}
\definecolor{cblue}{HTML}{1F77B4}
\definecolor{cgreen}{HTML}{1A7F4E}
\definecolor{cgrey}{rgb}{0.6,0.6,0.6}
\definecolor{corange}{HTML}{F97B4F}
\definecolor{cpurple}{HTML}{6A3EA8}

\hypersetup{colorlinks=true, allcolors=cblue}

\newcommand{\ours}[0]{\textsc{AQC}}
\newcommand{\ourslong}[0]{Adaptive Q-Chunking}
\newcommand{\hourgreen}[1]{\textbf{\color{ourcolor}{#1}}}

\newcommand{\qc}[0]{\textsc{QC}}
\newcommand{\dqc}[0]{\textsc{DQC}}

\newcommand{\hourbase}[1]{{\color{cpurple}\textbf{#1}}}       % standard baselines
\newcommand{\hourqc}[1]{{\color{qccolor}\textbf{#1}}}       % QC-family methods

\newcommand{\cellbase}[1]{\cellcolor{cpurple!10}#1}
\newcommand{\cellqc}[1]{\cellcolor{qccolor!10}#1}
\newcommand{\cellours}[1]{\cellcolor{ourcolor!10}#1}

\newcommand{\hourgr}[1]{{\color{black!55}\textbf{#1}}}
\newcommand{\hourfbc}[1]{{\color{black!70}\textbf{#1}}}
\newcommand{\cellgr}[1]{\cellcolor{gray!10}#1}
\newcommand{\cellfbc}[1]{\cellcolor{gray!14}#1}

\newcommand{\cellbaseOverall}[1]{\cellcolor{cpurple!18}#1}
\newcommand{\cellqcOverall}[1]{\cellcolor{qccolor!16}#1}
\newcommand{\celloursOverall}[1]{\cellcolor{ourcolor!16}#1}

\newcommand{\ac}[2]{\mathbf{a}_{#1:#2}}

\newcommand{\rrc}[2]{\mathbf{r}_{#1:#2}}

\newcommand{\Kset}{\mathcal{K}}

\newcommand{\kstar}{k^*}

\newcommand{\Qk}[1]{Q^{#1}}
\newcommand{\Vk}[1]{V^{#1}}

\newcommand{\Qkbar}[1]{\bar{Q}^{#1}}
\newcommand{\Vkbar}[1]{\bar{V}^{#1}}

\newcommand{\score}[2]{\mathrm{score}(#1, #2)}
\newcommand{\tildescore}[2]{\tilde{\mathrm{score}}(#1, #2)}

\newcommand{\expectile}[1]{f^{\kappa_V}_{\mathrm{exp}}\!\left(#1\right)}

\newcommand{\Rsum}[1]{\sum_{j=0}^{#1-1} \gamma^{j} r_{t+j}}

\newcommand{\pibeta}{\pi_{\beta}}

\newcommand{\emaqmax}[1]{\max_{i \leq N} \Qkbar{h}(s_{t+h}, a^{(i)}_{t+h:t+2h})}

\newcommand{\kdag}{k^\dagger}

\newcommand{\vcl}{V^{\bullet}}            % Closed-loop AQC value
\newcommand{\vaqc}{V^{\mathrm{AQC}}}      % AQC effective policy value
\newcommand{\vdag}{V^{\dagger}}           % Oracle selector value

\newcommand{\diamA}{\mathrm{diam}(A)}     % Advantage range
\newcommand{\Rbar}{\bar{R}}               % Max advantage magnitude

\newcommand{\te}[1]{\texttt{#1}}

\tcbset{
  defbox/.style={
    enhanced,
    colback=darkblue!6,
    colframe=darkblue,
    boxrule=0.5pt,
    arc=3pt,
    left=5pt, right=5pt, top=3pt, bottom=3pt,
    breakable,
    fonttitle=\bfseries,
  }
}

\tcbset{
  theorembox/.style={
    enhanced,
    colback=cgreen!6,
    colframe=cgreen,
    boxrule=0.5pt,
    arc=3pt,
    left=5pt, right=5pt, top=3pt, bottom=3pt,
    breakable,
    fonttitle=\bfseries,
  }
}

\tcbset{
  corollarybox/.style={
    enhanced,
    colback=cpurple!6,
    colframe=cpurple,
    boxrule=0.5pt,
    arc=3pt,
    left=5pt, right=5pt, top=3pt, bottom=3pt,
    breakable,
    fonttitle=\bfseries,
  }
}

\tcbset{
  aqcbox/.style={
    enhanced,
    colback=ourlightcolor!50,
    colframe=ourcolor,
    boxrule=0.5pt,
    arc=3pt,
    left=5pt, right=5pt, top=3pt, bottom=3pt,
    breakable,
    fonttitle=\bfseries,
  }
}

\tcbset{
  assumptionbox/.style={
    enhanced,
    colback=ourlightcolor!50,
    colframe=ourcolor,
    boxrule=0.5pt,
    arc=3pt,
    left=5pt, right=5pt, top=3pt, bottom=3pt,
    breakable,
    fonttitle=\bfseries,
  }
}

\tcolorboxenvironment{definition}{defbox}
\tcolorboxenvironment{theorem}{theorembox}
\tcolorboxenvironment{corollary}{corollarybox}
\tcolorboxenvironment{proposition}{aqcbox}
\tcolorboxenvironment{assumption}{assumptionbox}

\newcommand{\tval}[1]{\makebox[2.2em][c]{#1}}
\newcommand{\tcell}[2]{\tval{#1}$\!\to\!$\tval{#2}}

%% file: sections/introduction.tex
\section{Introduction}
\label{sec:intro}

Reinforcement learning (RL) promises to learn any task from rewards alone, but starting from scratch is impractical for complex behaviors~\citep{levine2020offline}.
Offline RL addresses this by learning a fixed policy from pre-collected data, but must contend with distributional shift and value overestimation~\citep{kumar2020conservative, fujimoto2018addressing,ross2011reduction}.
A practical approach is offline-to-online RL: first, learn from a fixed dataset of demonstrations, then continue improving through online interaction~\citep{nair2020awac, lee2022offline, ball2023efficient, nakamoto2024cal, zhou2025wsrl}.
This combination works well for many tasks, but struggles in long-horizon settings where rewards are sparse~\citep{ogbench_park2024, park2025horizon}.
The agent must execute hundreds of actions before receiving feedback, making learning slow.
Even worse, single-step policies tend to behave inconsistently during exploration, rarely reaching the states where rewards occur~\citep{li2025reinforcement}.

Action chunking addresses both problems by committing to sequences of actions rather than single steps~\citep{zhao2023learning, chi2023diffusion}.
A growing body of work incorporates chunks into RL: transformer-based episodic critics~\citep{li2024top}, coarse-to-fine discrete Q-networks~\citep{seo2024reinforcement, seo2024continuous}, detached value learning for offline RL~\citep{kim2025deas}, and receding-horizon exploitation~\citep{nagy2026sear}.
Among these, Q-chunking (\qc{})~\citep{li2025reinforcement} is the most directly relevant predecessor: it learns Q-values over chunks of $h$ consecutive actions, making value estimation stable even when the dataset is imperfect~\citep{fedus2020revisiting}, and executing entire chunks open-loop produces more coherent exploration.
However, \qc{} uses a single fixed chunk size $h$ for every state and every task.
Decoupled Q-chunking (\dqc{})~\citep{li2026decoupled} partially addresses this rigidity by allowing the policy and critic to operate at different chunk sizes, but it does not adapt the chunk sizes per state and requires goal-conditioned training that does not apply to standard reward-based tasks~\citep{andrychowicz2017hindsight}.

\begin{figure}[t]
    \centering
    \includegraphics[width=0.95\linewidth]{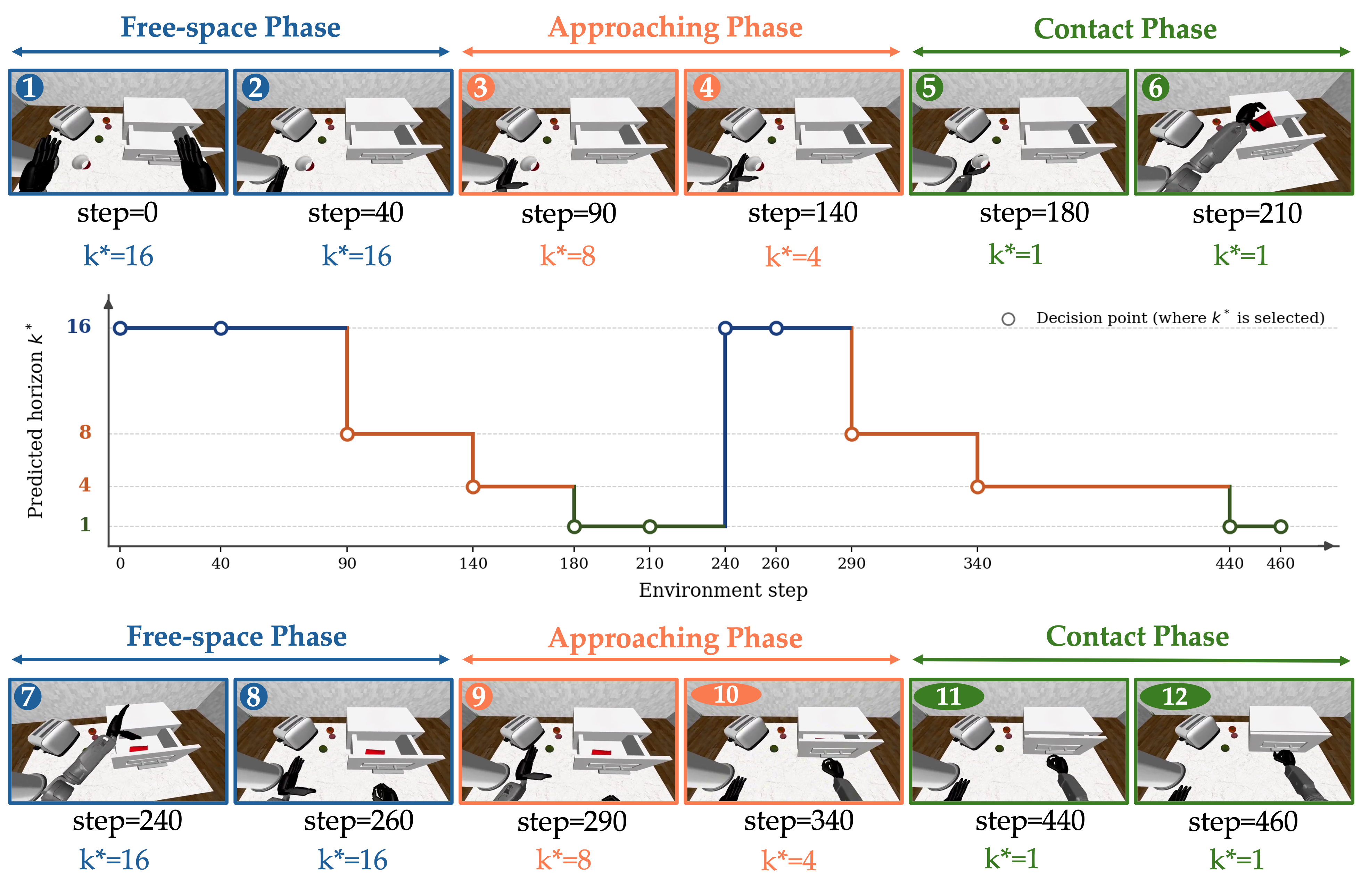}
    \caption{\small
        \textbf{Qualitative rollout of AQC.} Our method adaptively adjusts the commitment horizon $\kstar$ based on the task phase. It utilizes long chunks ($\kstar=16$) for efficient movement in free space and automatically switches to fine-grained control ($\kstar=1$) during complex contact-rich manipulation.
    }
    \label{fig:teaser}
    \vspace{-12pt}
\end{figure}

We argue that the deeper limitation is to use any fixed chunk size at all.
Consider a robot arm moving through free space toward an object: a long action sequence works well here because the motion is predictable.
But near a contact event like grasping or insertion, the same long sequence becomes harmful---small errors accumulate through contact dynamics, and the robot must react to new observations at every step.
No single chunk size handles both situations well, and tuning the right size for each task requires extensive experimentation.

The natural solution is to train critics for several chunk sizes and pick the best one at each state.
But this naive approach fails for two reasons.
First, critics for shorter chunks produce larger values by construction, so the agent always picks the shortest chunk, even when a longer one would be better.
Second, in states with low expected return, all chunk sizes give nearly identical values, so the selection becomes random noise rather than a meaningful choice.

We propose \hourgreen{\ourslong{} (\hourgreen{\ours{}})}, which resolves both problems by comparing the \emph{advantage} of each chunk size relative to what the baseline behavior would achieve.
The advantage measures how much better a particular chunk is compared to executing a chunk of that length from the demonstration policy.
This comparison removes both sources of bias: short chunks no longer dominate artificially, and in low-value states where no chunk is clearly better, the agent defaults to the longest available chunk.
When a short chunk is uniquely helpful near contacts, or a long chunk benefits free-space motion, the advantage becomes large and guides the selection (\cref{fig:teaser}).

The same approach naturally extends to enhancing vision-language-action (VLA) models~\citep{bjorck2025gr00t} with offline RL.
We train critics on top of a GR00T N1.6 backbone on RoboCasa-GR1 tabletop manipulation tasks.
The VLA natively outputs action chunks; \hourgreen{\ours{}} scores multiple candidate chunks at each scale and selects the best prefix via the per-scale advantage criterion.
This gives the VLA reactive control near contacts while preserving smooth motion in free space, all without modifying the underlying model.

Our main contributions are:
\begin{itemize}[leftmargin=*, topsep=2pt, itemsep=1pt]
    \item We present \hourgreen{\ours{}}, an offline-to-online RL method that adapts chunk sizes per state through a principled advantage-based criterion.
    \item We provide theoretical analysis establishing advantage separability, value dominance over fixed-chunk policies, and closed-loop optimality bounds.
    \item We demonstrate state-of-the-art results on OGBench~\citep{ogbench_park2024} and Robomimic~\citep{robomimic2021}, outperforming prior methods across diverse domains.
    \item We show that our approach enhances the performance of VLA policies on RoboCasa-GR1 tabletop tasks, exceeding prior chunked-critic methods on contact-rich manipulation.
\end{itemize}

%% file: sections/related.tex
\section{Related Work}
\label{sec:related}
\vspace{-5pt}
\paragraph{Offline-to-online reinforcement learning.}
Offline-to-online RL methods pretrain a policy on a fixed dataset and subsequently refine it through environment interaction.
Prior work addresses the offline-to-online value shift via advantage-weighted regression~\citep{nair2020awac}, implicit Q-learning~\citep{kostrikov2021offline}, calibrated pessimism~\citep{nakamoto2024cal}, interleaved offline-online batches~\citep{ball2023efficient}, warmup-based recalibration~\citep{zhou2025wsrl}, policy expansion~\citep{zhang2023policy}, or reincarnating RL~\citep{agarwal2022reincarnating}.
While these methods focus on managing the offline-to-online transition, our approach instead targets a complementary bottleneck: the inability of existing methods to adjust their planning horizon per state.

\paragraph{Action chunking in reinforcement learning.}
Action chunking predicts and executes sequences of future actions rather than a single action at each step.
\qc{}~\citep{li2025reinforcement} extends this to TD-based RL, training a critic over full $h$-step chunks to eliminate the off-policy multi-step bias while producing temporally coherent exploration.
TOP-ERL~\citep{li2024top} introduces transformer-based off-policy episodic RL with multi-step returns over short action chunks.
\citet{tian2025chunking} combine chunked critics with soft actor-critic, learning a transformer critic on $n$-step returns while keeping a single-step actor.
CQN-AS~\citep{seo2024reinforcement} and CQN~\citep{seo2024continuous} use coarse-to-fine Q-networks with action sequences for data-efficient RL, factorizing the Q-function over discrete action bins.
DEAS~\citep{kim2025deas} detaches value learning from the action sequence for scalable offline RL.
SEAR~\citep{nagy2026sear} exploits the temporal structure of action chunks with a receding horizon to improve sample efficiency.
\dqc{}~\citep{li2026decoupled} decouples the chunk size used for the critic from that used for the policy, but requires goal-conditioned hindsight relabeling~\citep{andrychowicz2017hindsight} and an additional distillation hyperparameter.
We instead learn partial critics via direct $k$-step TD losses bootstrapped from a long-horizon value estimate, achieving the same fixed point without distillation or goal conditioning, and extends to adaptive chunk size selection at inference.

\paragraph{Multi-step returns and hierarchical RL.}
Multi-step TD methods~\citep{hessel2018rainbow} accelerate credit propagation but introduce off-policy bias when data is stale~\citep{fedus2020revisiting, kozuno2021revisiting}.
Retrace~\citep{munos2016safe} corrects this bias with importance sampling truncation, while action chunking~\citep{li2025reinforcement} eliminates it by conditioning on the exact behavior-policy trajectory.
\citet{park2025horizon} show formally that longer effective horizons amplify bootstrapping error, motivating our approach of using a long-horizon value function to provide bootstrap targets for shorter-horizon critics.
Learning temporally extended actions has been widely studied in hierarchical RL~\citep{dayan1992feudal, dietterich2000hierarchical, kulkarni2016hierarchical, vezhnevets2016strategic, vezhnevets2017feudal, nachum2018data, ajay2020opal, shankar2020learning, pertsch2021accelerating, gehring2021hierarchical, xie2021latent} and the options framework~\citep{sutton1999between, menache2002q, csimcsek2004using, csimcsek2007betweenness, konidaris2011autonomous, bacon2017option, bagaria2019option, bagaria2024effectively}, though bi-level optimization remains a persistent challenge~\citep{nachum2018data}.
We collapse this bi-level problem into a single-level objective in a temporally extended action space, where the ``high-level'' selection is a deterministic advantage-based criterion requiring no additional learned parameters.
\vspace{-5pt}

%% file: sections/preliminaries.tex
\section{Background}
\label{sec:prelim}
\vspace{-5pt}
\textbf{Offline-to-online RL.}
We consider an infinite-horizon, fully observable Markov decision process (MDP)
$\mathcal{M} = (\mathcal{S}, \mathcal{A}, T, r, \rho, \gamma)$,
where $\mathcal{S}$ is the state space, $\mathcal{A}$ is the action space,
$T(s'\!\mid\!s,a)\!:\!\mathcal{S}\!\times\!\mathcal{A}\!\to\!\Delta(\mathcal{S})$ is the transition kernel,
$r(s,a)\!:\!\mathcal{S}\!\times\!\mathcal{A}\!\to\!\mathbb{R}$ is the reward function,
$\rho\in\Delta(\mathcal{S})$ is the initial state distribution, and $\gamma\in[0,1)$ is the discount factor.
We assume access to an offline dataset $\mathcal{D}=\{(s,a,r,s')\}$ collected by a behavior policy $\pibeta$.
The goal of offline-to-online RL is to find a policy $\pi\!:\!\mathcal{S}\!\to\!\Delta(\mathcal{A})$ that maximizes the expected discounted return
$\eta(\pi) := \mathbb{E}\!\left[\sum_{t=0}^{\infty} \gamma^t r(s_t, a_t)\right]$, operating in two phases: an \emph{offline phase} that pretrains on $\mathcal{D}$, followed by an \emph{online phase} that fine-tunes with environment interactions added to a growing replay buffer.

\textbf{TD learning and the bootstrapping bias problem.}
Actor-critic RL methods~\citep{sutton1998reinforcement, haarnoja2018soft} learn a critic $Q_\phi(s, a)$ via the temporal-difference (TD) loss:
\begin{align}
    \mathcal{L}(\phi) = \mathbb{E}_{s_t, a_t, r_t, s_{t+1} \sim \mathcal{D}}\!\left[\bigl(Q_\phi(s_t, a_t) - r_t - \gamma\bar{V}(s_{t+1})\bigr)^2\right],
\label{eq:td}
\end{align}
where $\bar{V}(s)$ is a target value estimate.
In a long-horizon, sparse-reward task with effective horizon $H = 1/(1-\gamma)$,
the 1-step backup propagates reward signal backward by only one step per gradient update,
requiring $O(H)$ updates before the initial state receives any signal.
Multi-step return methods~\citep{hessel2018rainbow} speed up this process by using a length-$n$ trajectory
segment $(s_t, a_t, \ldots, s_{t+n})$ to construct an $n$-step backup:
\begin{align}
    \hat{V}_{n\text{-step}} := \Rsum{n} + \gamma^n Q_{\bar\phi}(s_{t+n}, a_{t+n}),
    \quad a_{t+n} \sim \pi(\cdot \mid s_{t+n}),
    \label{eq:nstep}
\end{align}
reducing the effective horizon by a factor of $n$.
However, the $n$-step estimator is \emph{biased} when data is off-policy~\citep{fedus2020revisiting, kozuno2021revisiting}, meaning the cumulative rewards $\rrc{t}{t+n}$ in the replay buffer were collected under $\pibeta$, not the current policy $\pi$, so the sum is no longer an unbiased estimate of the on-policy $n$-step return.

\textbf{Action chunking Q-learning.}
\qc{}~\citep{li2025reinforcement} propose to resolve this tension with \emph{action chunking}: instead of training a single-action critic $Q(s_t, a_t)$, train a \emph{chunked critic}
$\Qk{h}(s_t, \ac{t}{t+h})$ that takes an entire chunk of $h$ consecutive actions as input.
The chunk TD backup is
\begin{align}
    \mathcal{L}_h(\phi) = \mathbb{E}\!\left[\!\left(\Qk{h}(s_t, \ac{t}{t+h}) - \Rsum{h} - \gamma^h \max_{i \leq N}\Qkbar{h}(s_{t+h}, a^{(i)}_{t+h:t+2h})\right)^{\!2}\right],
    \label{eq:qc}
\end{align}
where $\{a^{(i)}_{t+h:t+2h}\}_{i=1}^N$ are $N$ candidate action chunks sampled from $\pibeta$ at $s_{t+h}$.
The $\max$ over $N$ samples, known as EMAQ~\citep{ghasemipour2021emaq}, approximates the best-of-$N$ bootstrap and ensures the critic pushes above the behavior value.
Crucially, because the full chunk $\ac{t}{t+h}$ is taken directly from the data and held fixed during the backup,
the $n$-step bias is \emph{eliminated}: there is no policy mismatch between the actions that generated the rewards $\rrc{t}{t+h}$ and the actions that the critic conditions on~\citep{li2024top, li2025reinforcement}.

\textbf{Implicit value backup via expectile regression.}
To avoid explicit action maximization in the TD target, one can learn a value function
$V_\xi(s)$ to implicitly approximate $\max_a Q(s, a)$ using the \emph{expectile loss}~\citep{kostrikov2021offline}:
\begin{align}
    \mathcal{L}_V(\xi) = \mathbb{E}_{(s_t, \ac{t}{t+h})\sim\mathcal{D}}\!\left[
        \expectile{\Qkbar{h}(s_t, \ac{t}{t+h}) - V_\xi(s_t)}
    \right],
    \label{eq:expectile}
\end{align}
where $f^{\kappa_V}_{\mathrm{exp}}(u) = \lvert\kappa_V - \mathbf{1}_{[u < 0]}\rvert \, u^2$.
At the optimum, $V_\xi(s)$ approximates the $\kappa_V$-expectile of $Q(s, \cdot)$ over the data distribution;
with $\kappa_V$ close to $1$, this gives an upper envelope that approximates $V^*(s)$ under the EMAQ-boosted critic.

\textbf{Flow-matching behavior cloning.}
We parameterize $\pibeta$ as a flow-matching model~\citep{lipman2023flow} following FQL~\citep{park2025flow},
trained with the objective:
\begin{align}
    \mathcal{L}_{\mathrm{BC}}(\theta) = \mathbb{E}_{(s_t, \ac{t}{t+h})\sim\mathcal{D},\, x_0\sim\mathcal{N}(0,I),\, \tau\sim\mathcal{U}[0,1]}\!\left[
        \left\|v_\theta(s_t, x_\tau, \tau) - (\ac{t}{t+h} - x_0)\right\|^2
    \right],
    \label{eq:bc}
\end{align}
where $x_\tau = (1-\tau)x_0 + \tau\,\ac{t}{t+h}$ is the linear interpolation along the flow path.
At inference, $N$ candidate action chunks are drawn by solving the flow ODE from $x_0\sim\mathcal{N}(0,I)$,
and the best is selected by the critic.
This expressive policy class captures the multimodal, temporally correlated structure of manipulation datasets far better than Gaussian policies~\citep{chi2023diffusion, li2025reinforcement}.

%% file: sections/motivation.tex
\section{The Case for Adaptive Action Chunking}
\label{sec:motivation}

Both \hourqc{QC}~\citep{li2025reinforcement} and \hourqc{DQC}~\citep{li2026decoupled} fix the policy chunk size uniformly across all states.
We argue that this is suboptimal: the beneficial chunk size varies across states within a single trajectory, and a globally fixed size is always a compromise.
We further show that the natural approach to adapting the chunk size per state i.e. selecting by comparing $Q^k$ values directly is fundamentally broken.

\subsection{State-Dependent Commitment Length}
\label{sec:motivation-state-dep}

Different phases of a manipulation task require qualitatively different chunk sizes. During free-space motion e.g., reaching toward a target cube, executing a long action sequence open-loop is both safe and effective. The BC flow policy produces temporally coherent trajectories, and longer open-loop execution accelerates credit assignment by propagating value over more timesteps per update. Near a contact event such as grasping or placement, the same long-sequence execution becomes harmful. Small perturbations at a contact boundary propagate irreversibly through the subsequent physics, so the policy must react to state feedback at each step. A short chunk size is appropriate here.

A globally fixed chunk size cannot simultaneously serve both phases. A large chunk degrades performance near contacts by ignoring state feedback. A small chunk wastes the credit-assignment benefit of multi-step backups during free-space motion and yields incoherent online exploration.
This motivates maintaining a discrete set $\Kset = \{k_1, \ldots, k_{|\Kset|}\}$ of candidate chunk sizes and selecting the commitment length $\kstar(s_t) \in \Kset$ adaptively at each state.

\subsection{Why Direct \texorpdfstring{$Q^k$}{Q\textsuperscript{k}} Comparison Fails}
\label{sec:motivation-failure}

The natural adaptive selection rule compares the best action chunk across all candidate sizes using the learned critics:
\begin{align}
    \kstar(s_t),\; a^*
    \;=\; \underset{k \in \Kset,\; \ac{t}{t+k} \sim \pibeta^N}{\arg\max}\; \Qk{k}(s_t, \ac{t}{t+k}).
    \label{eq:naive-selector}
\end{align}
This selector fails due to two compounding problems:

\paragraph{Discount-scale mismatch.}
In sparse-reward tasks, intermediate rewards are nearly zero for the vast majority of transitions.
The Bellman equation for $\Qk{k}$ therefore simplifies to
\begin{align}
    \Qk{k}(s_t, \ac{t}{t+k})
    \;=\; \underbrace{\sum_{j=0}^{k-1} \gamma^j r_{t+j}}_{\approx\; 0}
    \;+\; \gamma^k \Vk{h}(s_{t+k})
    \;\approx\; \gamma^k \Vk{h}(s_{t+k}).
    \label{eq:qk-approx}
\end{align}
Since $\gamma < 1$, the factor $\gamma^k$ is strictly decreasing in $k$.
Consequently, $\Qk{k_1} > \Qk{k_2} > \cdots > \Qk{h}$ for nearly every state, regardless of which chunk size actually yields a better policy.
The selector degenerates to always choosing the smallest $k$ in $\Kset$.

\paragraph{State-dependent baseline mismatch.}
Dividing by $\gamma^k$ to remove the discount factor gives:
\begin{align}
    \underset{k}{\arg\max}\; \frac{\Qk{k}(s_t, \ac{t}{t+k})}{\gamma^k}
    \;\approx\;
    \underset{k}{\arg\max}\; \Vk{h}(s_{t+k}).
    \label{eq:discount-corrected}
\end{align}
This corrects the discount-scale issue but exposes a subtler one.
In states far from any reward, which constitute the majority of states in sparse-reward tasks, $\Vk{h}(s_{t+k}) \approx \epsilon$ for all $k \in \Kset$, where $\epsilon$ is small.
The differences across chunk sizes are then dominated by function approximation errors in the $\Qk{k}$ networks rather than by genuine differences in planning quality.
The argmax reflects network noise, not which chunk size is actually better.

Both failure modes stem from the same root cause: comparing raw $Q^k$ values lacks a \emph{per-scale reference} that accounts for the baseline value achievable at that chunk size under the data distribution.
We resolve this in \cref{sec:method} by subtracting a per-scale value baseline $\Vk{k}(s_t)$ and normalizing by $\gamma^k$, yielding an advantage criterion that avoids both collapse and noise amplification.

%% file: sections/method.tex
\section{Adaptive Q-Chunking}
\label{sec:method}

\hourgreen{\ours{}} selects the commitment length $\kstar(s_t) \in \Kset$ adaptively at each state, re-querying the policy every $\kstar(s_t)$ steps rather than uniformly every $h$ steps.
We first derive the selection criterion from first principles, then describe how its components are trained jointly.

\begin{figure}[t]
    \centering
    \includegraphics[width=\linewidth]{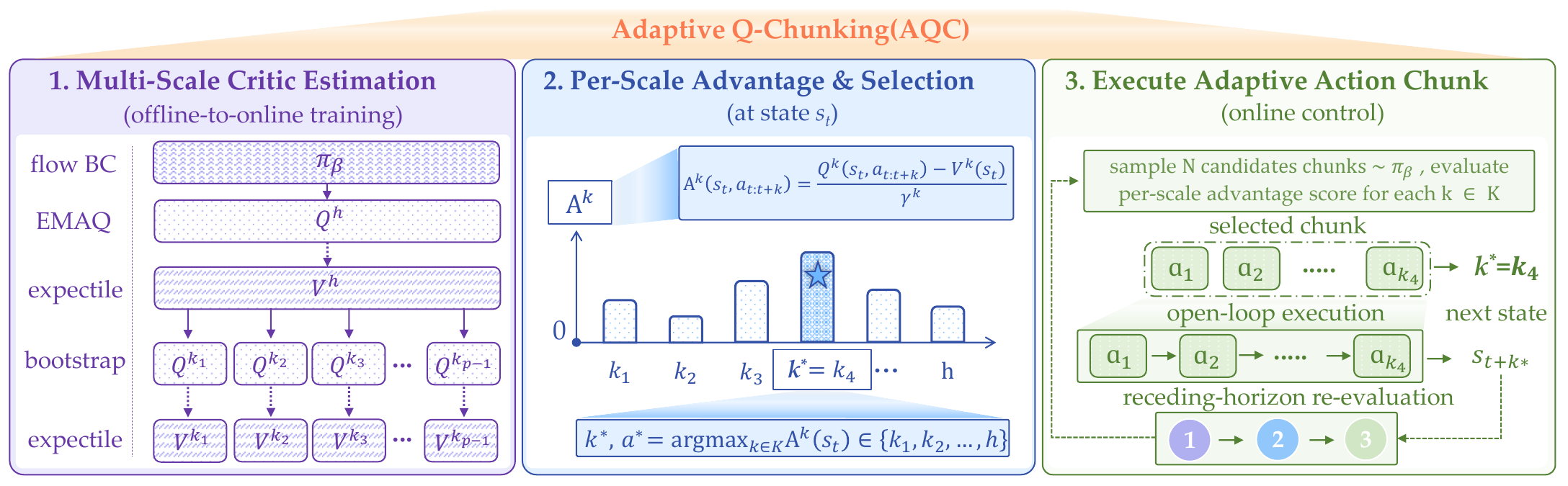}
    \caption{\small
        \textbf{Overview of \hourgreen{\ours{}}.}
        (1) Train a flow-BC policy $\pi_\beta$ alongside a long-horizon critic $Q^h$ and per-scale partial critics $\{Q^{k_i}\}$ bootstrapped from $V^h$.
        (2) Score $N$ candidate chunks at every horizon $k$ via the discount-normalized advantage, z-score normalized within each scale.
        (3) Select the best $(k^*, a^*)$ and execute open-loop for $k^*$ steps in a receding-horizon loop.
        Training and inference details are given in \cref{sec:method-training,sec:method-inference}. Here $p$ denotes $|\Kset|-1$
    }
    \label{fig:pipeline}
    \vspace{-4pt}
\end{figure}

\subsection{The per-scale advantage criterion}
\label{sec:method-criterion}

Suppose we maintain critics $\{\Qk{k}\}_{k \in \Kset}$ for a discrete set of chunk sizes $\Kset = \{k_1, \ldots, h\}$, and a behavior policy $\pibeta$ that can generate candidate action chunks of length up to $h$.
The ideal selection rule would pick the chunk size and action sequence with the best true future return:
\begin{align}
    \kstar(s_t),\; a^*
    \;\in\; \underset{k \in \Kset,\; \ac{t}{t+k} \sim \pibeta}{\arg\max}\;\; Q^{k,*}(s_t, \ac{t}{t+k}),
    \label{eq:ideal-selector}
\end{align}
where $Q^{k,*}$ is the optimal $k$-step action-value function.
Since we only have learned approximations $\Qk{k}$ and comparing their raw values across different $k$ systematically fails (\cref{sec:motivation-failure}), we instead compare the \emph{discount-normalized advantage} relative to a per-horizon baseline:
\begin{align}
    \score{k}{\ac{t}{t+k}}
    \;:=\;
    \frac{\Qk{k}(s_t, \ac{t}{t+k}) - \Vk{k}(s_t)}{\gamma^k}.
    \label{eq:advantage-selector}
\end{align}
Here $\Vk{k}(s_t)$ estimates the expected $k$-step return of the behavior policy from $s_t$.
The subtraction removes state-dependent scale at horizon $k$, and division by $\gamma^k$ removes the horizon-dependent discount.
We establish the noise-immunity of this criterion formally:

\begin{proposition}[Noise immunity of the advantage selector]
\label{prop:noise-immunity}
Let $\delta_k(s) := \Qk{k}(s, \ac{t}{t+k})/\gamma^k - \Vk{k}(s)/\gamma^k$ denote the discount-normalized advantage.
In a region where $\Vk{h}(s) \leq \epsilon$ for all reachable states, and assuming function approximation errors are bounded by $\sigma$,
\begin{align}
    |\delta_k(s)| \;\leq\; \epsilon + 2\sigma.
\end{align}
When $\epsilon \ll \sigma$, the advantage is dominated by approximation noise and all chunk sizes score near-zero.
By contrast, the uncorrected selector $\argmax_k \Qk{k}/\gamma^k$ lacks this safeguard: when all scores are $\approx \epsilon + \sigma_k$, the argmax picks the scale with the largest positive noise $\sigma_k$, producing a deterministic but systematically biased choice.
\end{proposition}

The key implication: the advantage criterion converts a biased wrong answer into an unbiased near-random choice when no genuine signal exists, which is strictly preferable.
A full proof appears in Appendix \ref{proof:noise-immunity}.

\subsection{Training}
\label{sec:method-training}

To instantiate \cref{eq:advantage-selector}, we need $\Qk{k}$ and $\Vk{k}$ for each $k \in \Kset$, trained jointly alongside $\pibeta$.

\paragraph{Long-horizon critic and value.}
The backbone of the pipeline is the long-horizon critic $\Qk{h}_\phi(s_t, \ac{t}{t+h})$, trained with the EMAQ~\citep{ghasemipour2021emaq} $h$-step TD loss:
\begin{align}
    \mathcal{L}_h(\phi) = \mathbb{E}_{\mathcal{D}}\!\left[\!\left(
        \Qk{h}_\phi(s_t, \ac{t}{t+h})
        - \Rsum{h}
        - \gamma^h \max_{i \leq N} \Qkbar{h}_\phi(s_{t+h}, a^{(i)}_{t+h:t+2h})
    \right)^{\!2}\right],
    \label{eq:qh-loss}
\end{align}
where $\{a^{(i)}_{t+h:t+2h}\}_{i=1}^N \sim \pibeta(\cdot \mid s_{t+h})$ and $\Qkbar{h}$ is the EMA target network.
Correspondingly, we fit $\Vk{h}_\xi(s)$ via expectile regression on $\Qk{h}$:
\begin{align}
    \mathcal{L}_{V^h}(\xi) = \mathbb{E}_{\mathcal{D}}\!\left[
        \expectile{\Qkbar{h}_\phi(s_t, \ac{t}{t+h}) - \Vk{h}_\xi(s_t)}
    \right].
    \label{eq:vh-loss}
\end{align}
With $\kappa_V\ \text{close to}\ 1$, $\Vk{h}_\xi$ approximates the upper envelope of $h$-step returns, serving both as the baseline for $k = h$ and as the bootstrap target for all partial critics.

\paragraph{Partial critics and per-scale baselines.}
For each $k \in \Kset \setminus \{h\}$, we train a partial critic $\Qk{k}_\psi(s_t, \ac{t}{t+k})$ that takes only the $k$-step prefix as input:
\begin{align}
    \mathcal{L}_k(\psi) = \mathbb{E}_{\mathcal{D}}
    \!\left[\!\left(
        \Qk{k}_\psi(s_t, \ac{t}{t+k})
        - \Rsum{k}
        - \gamma^k \Vkbar{h}_\xi(s_{t+k})
    \right)^{\!2}\right].
    \label{eq:qk-loss}
\end{align}
Using $\Vk{h}$ as the bootstrap target is a deliberate design choice.
The long-range value from $s_{t+k}$ onward is already encoded in $\Vk{h}$, so $\Qk{k}$ only needs to fit the $k$-step residual.
A 1-step value function would require chaining $O(H)$ TD updates to propagate credit, accumulating error at each step.
A formal bound showing that $V^h$ bootstrap yields tighter sub-optimality guarantees appears in Proposition \ref{prop:suboptimality-restate} (Proof \ref{proof:suboptimality-restate}).

Each $\Qk{k}$ is a separate network; it cannot be recovered by truncating $\Qk{h}$, since $\Qk{h}$ is conditioned on the full $h$-step chunk. We then fit per-scale baselines $\Vk{k}_\zeta(s)$ via expectile regression on the corresponding partial critics:
\begin{align}
    \mathcal{L}_{V^k}(\zeta) = \mathbb{E}_{\mathcal{D}}
    \!\left[
        \expectile{\Qkbar{k}_\psi(s_t, \ac{t}{t+k}) - \Vk{k}_\zeta(s_t)}
    \right].
    \label{eq:vk-loss}
\end{align}
These networks share the same architecture as $\Vk{h}_\xi$ and add negligible training overhead.

\paragraph{Behavior policy.}
We train $\pibeta$ with the flow-matching BC objective (\cref{eq:bc}) on the full $h$-step chunk distribution.

\subsection{Inference}
\label{sec:method-inference}

At each state $s_t$, \hourgreen{\ours{}} samples $N$ full-length candidate chunks from $\pibeta$ and for each chunk, evaluates the per-scale advantage score (Eq.\ref{eq:advantage-selector}) across all horizons $k \in \Kset$.
Because the raw advantage scores have different variance profiles across scales, we z-score normalize within each $k$ before comparison:
\begin{align}
    \tildescore{k}{a^{(i)}}
    \;=\;
    \frac{\score{k}{a^{(i)}} - \mathbb{E}_i[\score{k}{a^{(i)}}]}{\sqrt{\mathrm{Var}_i[\score{k}{a^{(i)}}]} + \epsilon}.
\end{align}
This prevents a scale with higher output variance from dominating the argmax simply due to spread rather than genuine advantage.
The selected chunk $a^*$ is executed open-loop for $\kstar$ steps, and the policy is re-queried at $s_{t+\kstar}$.
During online fine-tuning, new transitions are added to a replay buffer and all losses continue with mixed offline-online batches.

%% file: sections/experiments.tex
\input{tables/summary-table}

\begin{figure}[t]
    \centering
    \includegraphics[width=\linewidth]{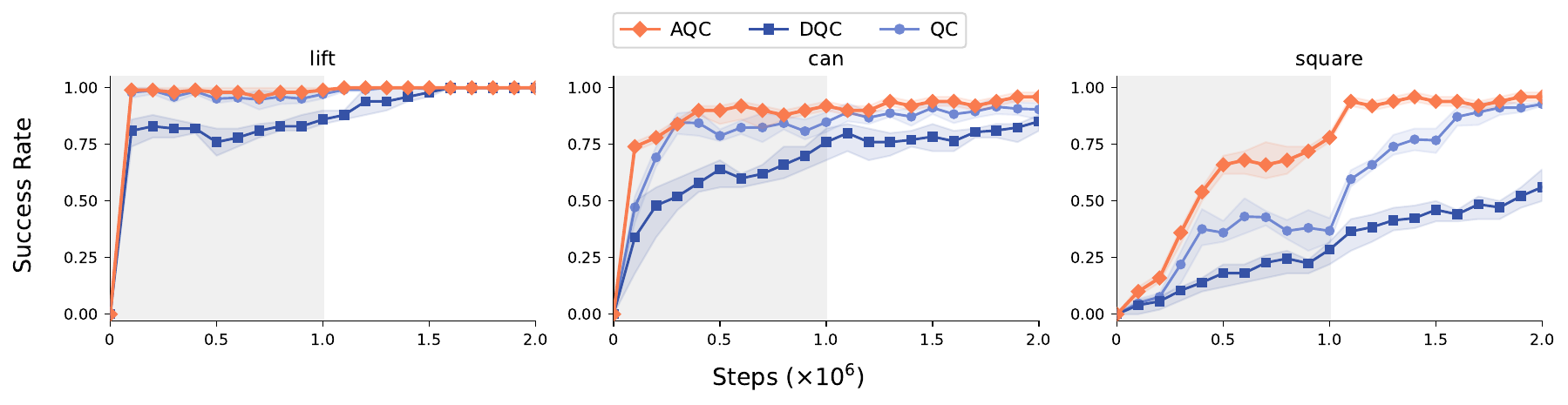}
    \caption{\small
        \textbf{Robomimic results.}
        Success rate vs.\ environment steps on three tasks. The first 1M steps are offline and the next 1M steps are online. (5 seeds)
    }
    \label{fig:robomimic}
\end{figure}

\section{Experimental Results}
\label{sec:exp}

We conduct experiments to analyze the empirical effectiveness of \hourgreen{\ours{}} on a range of long-horizon, sparse-reward domains.
In particular, we aim to answer the following questions:
\begin{enumerate}[start=1,label={(\bfseries Q\arabic*)},topsep=2pt,itemsep=0pt,leftmargin=*]
    \item \emph{How well does \hourgreen{\ours{}} compare to prior offline-to-online RL methods?}
    \item \emph{Can \hourgreen{\ours{}} enhance large-scale VLAs with offline RL?}
    \item \emph{Which components of \hourgreen{\ours{}} are most critical?}
\end{enumerate}

\subsection{Environments and Datasets}
\label{sec:exp-envs}

We consider long-horizon, sparse-reward manipulation domains across three benchmark suites.
For OGBench~\citep{ogbench_park2024}, we use a total of five domains. Each domain provides five tasks of increasing difficulty; we report success rates averaged across all five tasks and four seeds, with 95\% stratified bootstrap confidence intervals. We evaluate on \te{scene-sparse}, \te{puzzle-3x3-sparse}, \te{cube-double}, \te{cube-triple}, and \te{cube-quadruple}. We use datasets ranging from 1M to 100M transitions based on task difficulty.
For Robomimic~\citep{robomimic2021}, we use three tasks: \te{lift}, \te{can}, \te{square}; with multi-human datasets, evaluated over five seeds.
Additionally, we evaluate \hourgreen{\ours{}} by finetuning GR00T N1.6~\citep{bjorck2025gr00t} using offline RL on the 24 RoboCasa-GR1 tabletop manipulation tasks~\citep{nasiriany2024robocasa}. See Appendix \ref{appendix:exp} for more details about experiments setup and dataset.

\subsection{Comparisons}
\label{sec:exp-baselines}

We compare against prior methods that speed up value backup as well as the previous best offline-to-online RL methods: \hourbase{IQL}~\citep{kostrikov2021offline}, \hourbase{RLPD}~\citep{ball2023efficient}, \hourbase{FQL}~\citep{park2025flow}, \hourbase{FQL-n}, \hourqc{QC}, \hourqc{QC-FQL}~\citep{li2025reinforcement}, and \hourqc{DQC}~\citep{li2026decoupled}.

\subsection{How well does \hourgreen{\ours{}} compare to prior methods?}
\label{sec:exp-main} 

\cref{tab:main} reports offline and online success rates across all five domains of OGBench.
\hourgreen{\ours{}} achieves the highest online success rates on the three hardest domains (\te{cube-double}, \te{cube-triple}, \te{cube-quadruple}).
On the easier domains (\te{scene-sparse}, \te{puzzle-3x3}), all chunking methods converge to near-perfect online performance.
Across all domains, \hourgreen{\ours{}} consistently outperforms both \hourqc{QC} and \hourqc{DQC}, with the gap widening as horizon length increases. \cref{fig:robomimic} reports training curves on \te{lift}, \te{can}, and \te{square} from Robomimic.
\hourgreen{\ours{}} matches or exceeds \hourqc{QC} on all three tasks.
The gains are most pronounced on \te{square}, which involves a contact-rich peg-insertion step and directly illustrates the core motivation of our method: during the pre-grasp reaching phase, the motion is predictable and \hourgreen{\ours{}} selects longer chunks for temporally coherent trajectory; but once the nut contacts the peg, small misalignments require rapid corrective feedback, and \hourgreen{\ours{}} adaptively switches to shorter chunks to react to state changes. On \te{lift} and \te{can}, which involve simpler pick-and-place motions with fewer contact transitions, the advantage of adaptive selection is smaller but still non-negative, suggesting that even in these tasks there are phases where committing to shorter or longer chunks is beneficial.

\subsection{Can \hourgreen{\ours{}} enhance large-scale VLAs with offline RL?}
\label{sec:exp-hard}

\cref{fig:vla} evaluates \hourgreen{\ours{}} on RoboCasa-GR1 tabletop manipulation tasks~\citep{nasiriany2024robocasa}, a challenging multi-task setting that tests whether our method generalizes to large-scale vision-language-action models.
The 24 tasks span diverse contact-rich manipulation skills including picking, placing, opening containers, and operating household appliances.
These tasks are particularly demanding for RL fine-tuning: VLAs are heavily overparameterized, the state-action space is high-dimensional, and collecting additional online data is impractical, making effective offline RL from limited suboptimal rollouts essential.

We use GR00T N1.6~\citep{bjorck2025gr00t} as the base VLA, which autoregressively generates 16-step action chunks conditioned on the current observation.
Our protocol has two stages.
First, we finetune GR00T N1.6 via supervised learning on 24k expert demonstrations spanning all 24 tasks, yielding a multi-task behavior policy. We deploy this policy to collect 300 offline rollouts per task.
Then, we fine-tune the base policy using behavior cloning on the combined expert + rollout dataset and use the model as an actor for training critic functions when necessary (Filtered BC).
\vspace{-2pt}
\begin{wrapfigure}{r}{0.45\textwidth}
    \centering
    \includegraphics[width=\linewidth]{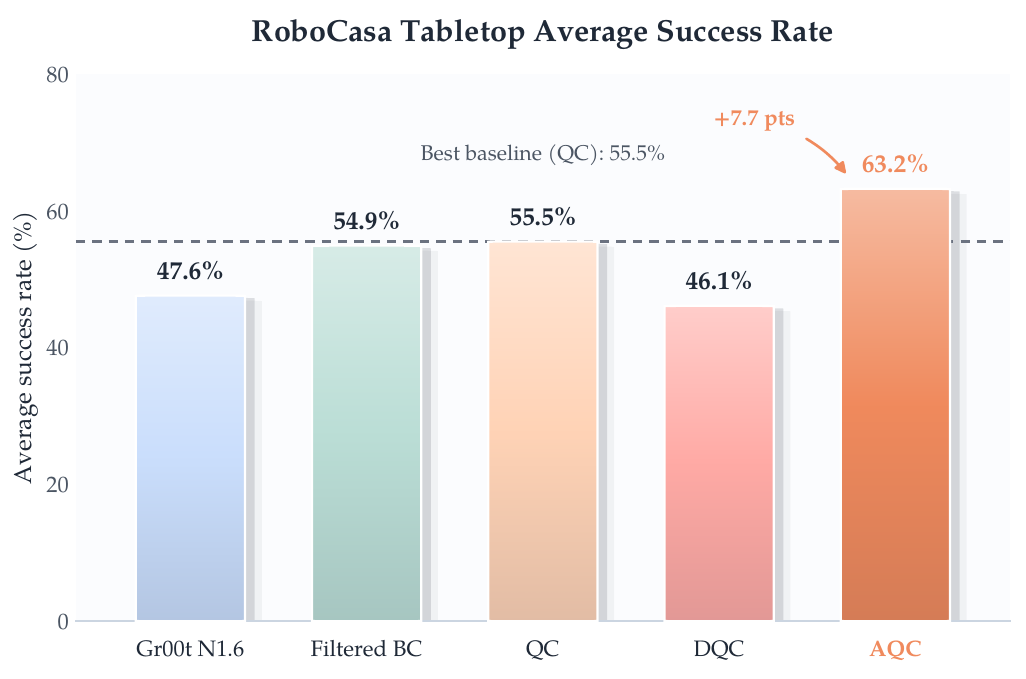}
    \caption{\footnotesize
        \textbf{Enhancing GR00T N1.6 on RoboCasa-GR1 tasks.}
        Success rates across 24 tabletop manipulation tasks.
        We compare the base VLA against Filtered BC, \hourqc{QC}, \hourqc{DQC}, and \hourgreen{\ours{}}.
    }
    \label{fig:vla}
    \vspace{-4pt}
\end{wrapfigure}

For \hourgreen{\ours{}}, we set $h = 16$ to match GR00T N1.6's native chunk size and use $\mathcal{K} = \{1, 4, 8, 16\}$ as candidate commitment horizons.
As shown in \cref{fig:vla}, \hourgreen{\ours{}} outperforms both \hourqc{QC} and \hourqc{DQC} across all 24 tasks, with the largest margins on tasks involving multi-stage contact transitions (e.g., opening a cabinet then manipulating its contents).
Filtered BC provides modest improvements over the base VLA but is fundamentally limited by its inability to leverage value information from suboptimal transitions.
These results suggest that \hourgreen{\ours{}} scales effectively to large policies and high-dimensional observations, and that adaptive chunk selection remains beneficial even when the underlying policy is a multi-task VLA rather than a single-task expert.

\subsection{What makes \hourgreen{\ours{}} work?}
\label{sec:exp-ablation}

\begin{figure}[h]
    \begin{minipage}[b]{0.49\linewidth}
        \centering
        \includegraphics[width=\linewidth]{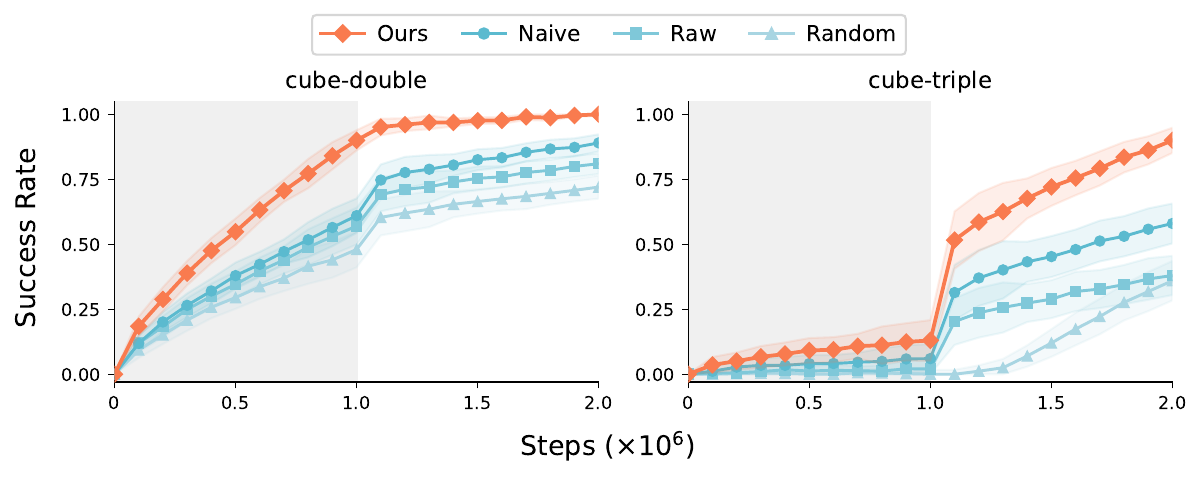}
        \caption{\small
            \textbf{Advantage criterion ablation} of selection criteria on \te{cube-double} (left) and \te{cube-triple} (right). Shaded regions are 95\% confidence intervals over 4 seeds.
        }
        \label{fig:ablation-criterion}
    \end{minipage}
    \hfill
    \begin{minipage}[b]{0.49\linewidth}
        \centering
        \includegraphics[width=\linewidth]{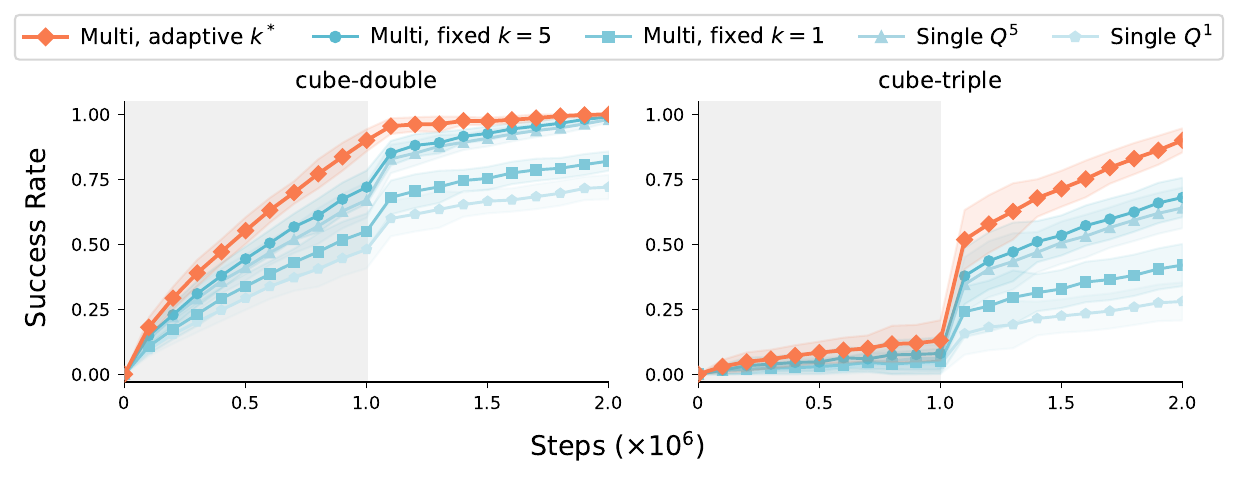}
        \caption{\small
            \textbf{Multi-scale critics and adaptive selection ablation} on \te{cube-double} (left), and \te{cube-triple} (right). Shaded regions are 95\% confidence intervals over 4 seeds.
        }
        \label{fig:ablation-adaptive}
    \end{minipage}
\end{figure}

We run ablations on \te{cube-double} and \te{cube-triple} with 4 seeds each. 
The core design in \hourgreen{\ours{}} is the per-scale advantage selector.
\cref{fig:ablation-criterion} compares it against discount-corrected Q values without baselines, raw Q values, and random selection.
The raw-Q variant collapses to always selecting $k = 1$, confirming the discount-scale collapse from \cref{sec:motivation}.
The discount-corrected variant partially recovers on easier tasks but degrades sharply on longer-horizon ones where function approximation noise dominates without a per-scale baseline. \cref{fig:ablation-adaptive} decomposes gains from (A) training multiple per-scale critics $Q^k$ and (B) adaptive $k^*$ selection at inference.
Multi-critic training alone improves performance over single-critic baselines, confirming that richer value representation helps even without adaptation.
Adaptive selection adds further gains on top, with the effect compounding on longer-horizon tasks where contact phases dominate. Additional ablations are in Appendix \ref{app:ablations}.

%% file: tables/summary-table.tex
\begin{table}[t]
    \centering
    \setlength{\tabcolsep}{3.5pt}
    \renewcommand{\arraystretch}{1.15}
    \resizebox{\textwidth}{!}{
    \begin{tabular}{l l c c c c c c}
        \toprule
        & & \te{puzzle-3x3} & \te{scene} & \te{cube-double} & \te{cube-triple} & \te{cube-quad.} & Overall \\
        & & (5 tasks) & (5 tasks) & (5 tasks) & (5 tasks) & (5 tasks) & (25 tasks) \\
        \midrule
        %% ---- 1-step TD ----
        \multirow{3}{*}{\small\textbf{1-step TD}}
        & \cellbase{\hourbase{IQL}~\citep{kostrikov2021offline}}
            & \cellbase{\tcell{0}{20}}
            & \cellbase{\tcell{0}{39}}
            & \cellbase{\tcell{0}{0}}
            & \cellbase{\tcell{0}{0}}
            & \cellbase{\tcell{0}{0}}
            & \cellbaseOverall{\tcell{0}{12}} \\
        & \cellbase{\hourbase{RLPD}~\citep{ball2023efficient}}
            & \cellbase{\tcell{--}{\textbf{100}}}
            & \cellbase{\tcell{--}{94}}
            & \cellbase{\tcell{--}{\textbf{99}}}
            & \cellbase{\tcell{--}{41}}
            & \cellbase{\tcell{--}{0}}
            & \cellbaseOverall{\tcell{--}{67}} \\
        & \cellbase{\hourbase{FQL}~\citep{park2025flow}}
            & \cellbase{\tcell{99}{\textbf{100}}}
            & \cellbase{\tcell{52}{\textbf{95}}}
            & \cellbase{\tcell{32}{76}}
            & \cellbase{\tcell{2}{18}}
            & \cellbase{\tcell{0}{3}}
            & \cellbaseOverall{\tcell{37}{58}} \\
        \arrayrulecolor{black!25}\midrule
        %% ---- n-step TD ----
        \multirow{1}{*}{\small\textbf{$n$-step TD}}
        & \cellbase{\hourbase{FQL-n}}
            & \cellbase{\tcell{99}{\textbf{100}}}
            & \cellbase{\tcell{21}{70}}
            & \cellbase{\tcell{9}{77}}
            & \cellbase{\tcell{1}{1}}
            & \cellbase{\tcell{7}{36}}
            & \cellbaseOverall{\tcell{27}{57}} \\
        \arrayrulecolor{black!25}\midrule
        %% ---- Q-chunking ----
        \multirow{3}{*}{\small\textbf{Q-chunking}}
        & \cellqc{\hourqc{QC-FQL}~\citep{li2025reinforcement}}
            & \cellqc{\tcell{64}{\textbf{100}}}
            & \cellqc{\tcell{86}{\textbf{99}}}
            & \cellqc{\tcell{42}{\textbf{100}}}
            & \cellqc{\tcell{3}{53}}
            & \cellqc{\tcell{2}{77}}
            & \cellqcOverall{\tcell{38}{86}} \\
        & \cellqc{\hourqc{QC}~\citep{li2025reinforcement}}
            & \cellqc{\tcell{\textbf{100}}{\textbf{100}}}
            & \cellqc{\tcell{82}{\textbf{99}}}
            & \cellqc{\tcell{66}{\textbf{98}}}
            & \cellqc{\tcell{5}{64}}
            & \cellqc{\tcell{3}{73}}
            & \cellqcOverall{\tcell{52}{86}} \\
        & \cellqc{\hourqc{DQC}~\citep{li2026decoupled}}
            & \cellqc{\tcell{86}{\textbf{97}}}
            & \cellqc{\tcell{56}{89}}
            & \cellqc{\tcell{15}{55}}
            & \cellqc{\tcell{1}{0}}
            & \cellqc{\tcell{2}{20}}
            & \cellqcOverall{\tcell{32}{52}} \\
        \arrayrulecolor{black}\midrule
        %% ---- AQC (Ours) ----
        \multirow{1}{*}{\small\textbf{Ours}}
        & \cellours{\hourgreen{\ours{}}}
            & \cellours{\tcell{\textbf{100}}{\textbf{100}}}
            & \cellours{\tcell{\textbf{98}}{\textbf{100}}}
            & \cellours{\tcell{\textbf{90}}{\textbf{100}}}
            & \cellours{\tcell{\textbf{13}}{\textbf{90}}}
            & \cellours{\tcell{\textbf{9}}{\textbf{88}}}
            & \celloursOverall{\tcell{\textbf{62}}{\textbf{96}}} \\
        \arrayrulecolor{black}\bottomrule
    \end{tabular}
}
    \vspace{2mm}
    \caption{\footnotesize
        \textbf{Offline-to-online RL results on OGBench.}
        Each cell reports \emph{offline$\,\to\,$online} success rate (\%) averaged over 5 tasks and 4 seeds.
        Numbers in bold are within 95\% of the best in each column.
        Full task-level breakdowns with mean and standard deviations are in Appendix \ref{app:full_results}.
    }
    \label{tab:main}
    \vspace{-4mm}
\end{table}

%% file: sections/conclusion.tex
\section{Conclusion}
\label{sec:conclusion}
\vspace{-5pt}
We proposed \hourgreen{\ourslong{} (\ours{})}, an offline-to-online RL method that adapts action chunk sizes per state via a discount-normalized advantage criterion. By comparing each candidate chunk size to the baseline policy at that horizon, \hourgreen{\ours{}} resolves the bias toward shorter chunks and the random selection in low-value states that plague naive multi-chunk methods. We proved advantage separability, value dominance over fixed-chunk policies, and closed-loop optimality bounds. Our empirical results demonstrate state-of-the-art performance across OGBench, Robomimic and RoboCasa-GR1. Beyond its immediate results, \hourgreen{\ours{}} suggests a broader principle for offline RL: when multiple learned value functions exist, whether for different horizons, abstractions, or representation levels, their raw values are incomparable without a per-scale reference. The advantage-based criterion provides a principled way to convert these incomparable quantities into a shared decision space. We hope this perspective inspires future work on adaptive commitment horizons and multi-scale value-based decision making.

%% file: sections/limitations.tex
\section{Limitations and Future Work}
\label{app:limitations}

Despite its empirical effectiveness, our method has several limitations that suggest directions for future work. First, the method selects from a fixed discrete candidate set $\Kset$, which requires some domain knowledge to construct and cannot discover chunk sizes outside the predefined set; learning $\Kset$ itself---for instance through a continuous relaxation of the commitment horizon or a state-conditioned proposal mechanism---would remove this design choice. Second, like all action-chunking methods, our approach executes selected chunks open-loop, meaning even $k{=}1$ chunks inherit the limited reactivity of the flow-BC policy $\pibeta$, which was trained to produce $h$-step sequences; integrating a closed-loop low-level executor for contact-rich phases could address this gap. Third, the critic training overhead scales linearly with $|\Kset|$, and while small in practice, weight-sharing schemes across scales such as a single multi-scale critic with horizon-embedding inputs could improve efficiency for large candidate sets. The selection also depends on the quality and diversity of $\pibeta$ as a candidate source and on z-score normalization across $N$ samples, both of which affect robustness when $N$ is small or the policy distribution shifts during online fine-tuning.

\section{Algorithm}
\begin{algorithm}[h]
\caption{\hourgreen{\ours{}}}
\label{algo:aqc}
\begin{algorithmic}[1]
\Require Chunk sizes $\Kset = \{k_1, \ldots, h\}$, samples $N$, expectile $\kappa_V$, warmup
\State Initialize critics $\Qk{h}_\phi$, $\{\Qk{k}_\psi\}_{k \in \Kset \setminus \{h\}}$
\State Initialize values $\Vk{h}_\xi$, $\{\Vk{k}_\zeta\}_{k \in \Kset \setminus \{h\}}$
\State Initialize flow policy $\pibeta$

\Statex
\Statex \textbf{Offline pretraining:}
\While{not converged}
    \State Sample minibatch $(s_t, \ac{t}{t+h}, \Rsum{h}, s_{t+h}) \sim \mathcal{D}$
    \State Update $\Qk{h}_\phi$ via \cref{eq:qh-loss} (EMAQ $h$-step TD)
    \State Update $\Vk{h}_\xi$ via \cref{eq:vh-loss} (expectile regression on $\Qk{h}$)
    \ForAll{$k \in \Kset \setminus \{h\}$}
        \State Update $\Qk{k}_\psi$ via \cref{eq:qk-loss} ($k$-step TD + $\Vk{h}$ bootstrap)
        \State Update $\Vk{k}_\zeta$ via \cref{eq:vk-loss} (expectile regression on $\Qk{k}$)
    \EndFor
    \State Update $\pibeta$ via flow-matching BC (\cref{eq:bc})
\EndWhile

\Statex
\Statex \textbf{Online fine-tuning:}
\State Initialize replay buffer $\mathcal{R} \gets \mathcal{D}$
\For{each environment step $t$}
    \State Observe $s_t$
    \State Sample $N$ full-length chunks $\{a^{(i)}_{1:h}\}_{i=1}^N \sim \pibeta(\cdot \mid s_t)$
    \ForAll{$k \in \Kset$}
        \State Score: $\score{k}{a^{(i)}_{1:k}} = \dfrac{\Qk{k}(s_t, a^{(i)}_{1:k}) - \Vk{k}(s_t)}{\gamma^k}$ for all $i$, using prefix $a^{(i)}_{1:k}$
    \EndFor
    \State Z-score normalize within each scale: $\tildescore{k}{} \gets (\score{k}{} - \mathrm{mean}_i)/(\mathrm{std}_i + \epsilon)$
    \State Select $k^*, a^* \gets \argmax_{k, i} \tildescore{k}{a^{(i)}_{1:k}}$
    \State Execute $a^*$ open-loop for $k^*$ steps
    \State Store transitions in $\mathcal{R}$
    \If{$t \geq$ warmup}
        \State Sample mixed batch from $\mathcal{R}$
        \State Update all networks via losses \cref{eq:qh-loss,eq:vh-loss,eq:qk-loss,eq:vk-loss} and \cref{eq:bc}
    \EndIf
\EndFor
\end{algorithmic}
\end{algorithm}

%% file: sections/appendix_experiments.tex
\section{Experimental Setup}
\label{appendix:exp}

\subsection{OGBench environments}

We consider five manipulation domains from OGBench~\citep{ogbench_park2024} as shown in \cref{fig:ogbench_tasks}, same as the ones used in ~\citet{li2025reinforcement}: \te{scene-sparse}, \te{puzzle-3x3-sparse}, \te{cube-double}, \te{cube-triple}, and \te{cube-quadruple}. Each domain provides five evaluation tasks of increasing difficulty; we report success rates averaged across all five tasks. We use ``play''-style datasets collected by non-Markovian expert policies with temporally correlated noise. We use the default 1M-transition datasets except for \te{cube-quadruple}, where we use the 100M-transition dataset. All domains use binary or near-binary sparse reward functions and a UR5e robot arm with 5D end-effector control.

\begin{figure}[h]
    \centering
    \includegraphics[width=\linewidth]{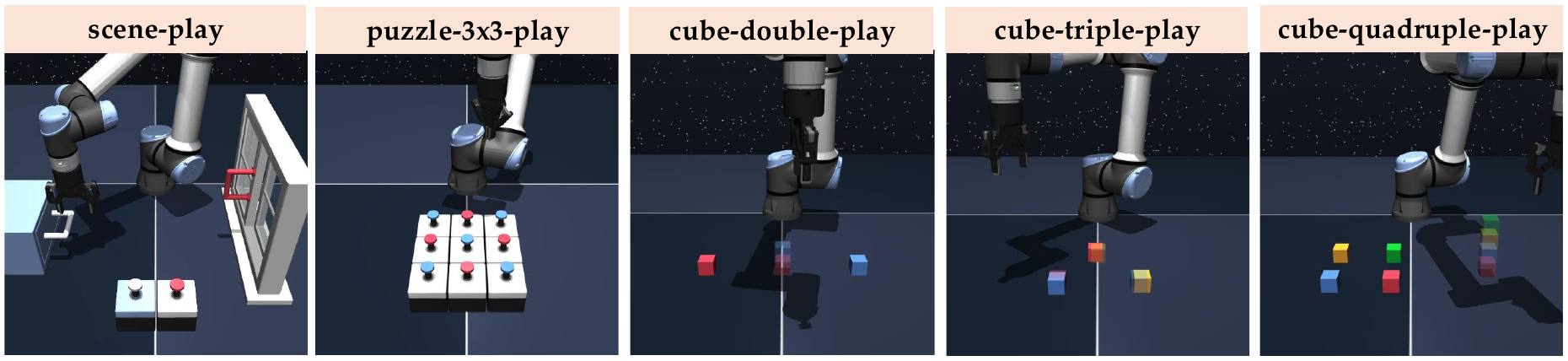}
    \caption{\small
        \textbf{OGBench tasks.} Some example tasks from the OGBench benchmark that we consider in our work.
    }
    \label{fig:ogbench_tasks}
\end{figure}

For \te{scene-sparse} and \te{puzzle-3x3-sparse}, we sparsify the reward function such that the agent receives $-1$ when the task is incomplete and $0$ upon completion. For the cube domains, the reward is $-n_{\mathrm{wrong}}$ where $n_{\mathrm{wrong}}$ is the number of cubes at incorrect positions; the episode terminates when all cubes are correctly placed (reward $0$). We now describe each domain briefly:

\textbf{Cube domains.} The agent controls a UR5e arm to pick and place cubes into target configurations. \te{cube-double} requires arranging two cubes, \te{cube-triple} three cubes, and \te{cube-quadruple} four. These domains test multi-object manipulation and long-horizon sequential reasoning. The cube domains become dramatically harder as the number of cubes increases, since permuting multiple blocks requires non-trivial sequential planning.

\textbf{Scene domain.} This domain involves manipulating everyday objects --- a drawer, a window, a cube, and two button locks --- into a target configuration. Pressing a button toggles whether the corresponding object can be moved. Evaluation tasks require multi-step sequences such as unlocking the drawer, opening it, placing a cube inside, and closing it. We use a sparsified reward variant where success yields $0$ and failure yields $-1$.

\textbf{Puzzle domains.} These domains implement the ``Lights Out'' puzzle game using a robot arm that must physically press buttons on a grid. Pressing a button toggles its color and the colors of adjacent buttons. The goal is to reach a target color configuration. We use sparsified rewards: $-1$ until the target configuration is reached, $0$ upon success.

\begin{table}[h]
    \centering
    \begin{tabular}{@{}l c c c@{}}
        \toprule
        \textbf{Domain} & \textbf{Dataset Size} & \textbf{Episode Length} & \textbf{Action Dim.} \\
        \midrule
        \te{scene-sparse-*}       & 1M   & 750  & 5 \\
        \te{puzzle-3x3-sparse-*}  & 1M   & 500  & 5 \\
        \te{cube-double-*}        & 1M   & 500  & 5 \\
        \te{cube-triple-*}        & 3M   & 1000 & 5 \\
        \te{cube-quadruple-100M-*}& 100M & 1000 & 5 \\
        \bottomrule
    \end{tabular}
    \vspace{2mm}
    \caption{\small \textbf{OGBench domain metadata.} Dataset sizes, episode lengths, and action dimensions for all five domains. Each domain provides five evaluation tasks (indicated by \texttt{*}).}
    \label{tab:ogbench-metadata}
\end{table}

\subsection{Robomimic environments}

We evaluate on three tasks from the robomimic benchmark~\citep{robomimic2021}, using the multi-human datasets collected by six human operators. Each dataset contains 300 successful demonstrations. All three tasks use binary task-completion rewards: $-1$ when the task is incomplete and $0$ upon success.

\begin{itemize}
    \item \te{lift}: The robot arm must grasp a small cube and lift it. This is the simplest task in the benchmark, testing basic grasping and vertical motion.
    \item \te{can}: The robot arm must pick up a soda can and place it into a smaller container bin. This requires more precise manipulation than \te{lift}, as the agent must orient the can correctly for insertion.
    \item \te{square}: The robot arm must pick up a square nut and thread it onto a peg. The nut is slightly larger than the peg, requiring precise alignment and contact-rich insertion, making this the most difficult of the three tasks.
\end{itemize}

\begin{table}[h]
    \centering
    \begin{tabular}{@{}l c c c@{}}
        \toprule
        \textbf{Task} & \textbf{Dataset Size} & \textbf{Episode Length} & \textbf{Action Dim.} \\
        \midrule
        \te{lift}   & 31{,}127 & 500  & 7 \\
        \te{can}    & 62{,}756 & 500  & 7 \\
        \te{square} & 80{,}731 & 500  & 7 \\
        \bottomrule
    \end{tabular}
    \vspace{2mm}
    \caption{\small \textbf{Robomimic task metadata.} Number of transitions, episode lengths, and action dimensions.}
    \label{tab:robomimic-metadata}
\end{table}

% \paragraph{Analysis of adaptive chunk selection on \te{square}.}
% The \te{square} task provides the clearest demonstration of \hourgreen{\ours{}}'s contact-aware behavior. A typical episode consists of three phases: (i) approaching the square nut, (ii) grasping and lifting, and (iii) threading the nut onto the peg. Phase (iii) is the critical contact-rich insertion step where the nut must align with the peg and push through with force feedback. Fixed-chunk methods are forced to commit to a single horizon: $h=1$ requires step-by-step execution throughout, losing temporal coherence during free-space motion; $h=10$ blindly executes long sequences, leading to jamming when the nut contacts the peg at a misaligned angle. In contrast, \hourgreen{\ours{}} learns to select longer chunks ($k \in \{5, 10\}$) during phases (i) and (ii) where the motion is predictable and benefits from temporal smoothing, then switches to shorter chunks ($k = 1$) during phase (iii) where reactive control is essential. This matches the intuition articulated in \cref{sec:intro}: no single chunk size handles both free-space reaching and contact insertion well, and per-state selection resolves this trade-off automatically.

\subsection{RoboCasa-GR1 environments}

\begin{figure}[h]
    \centering
    \includegraphics[width=0.6\linewidth]{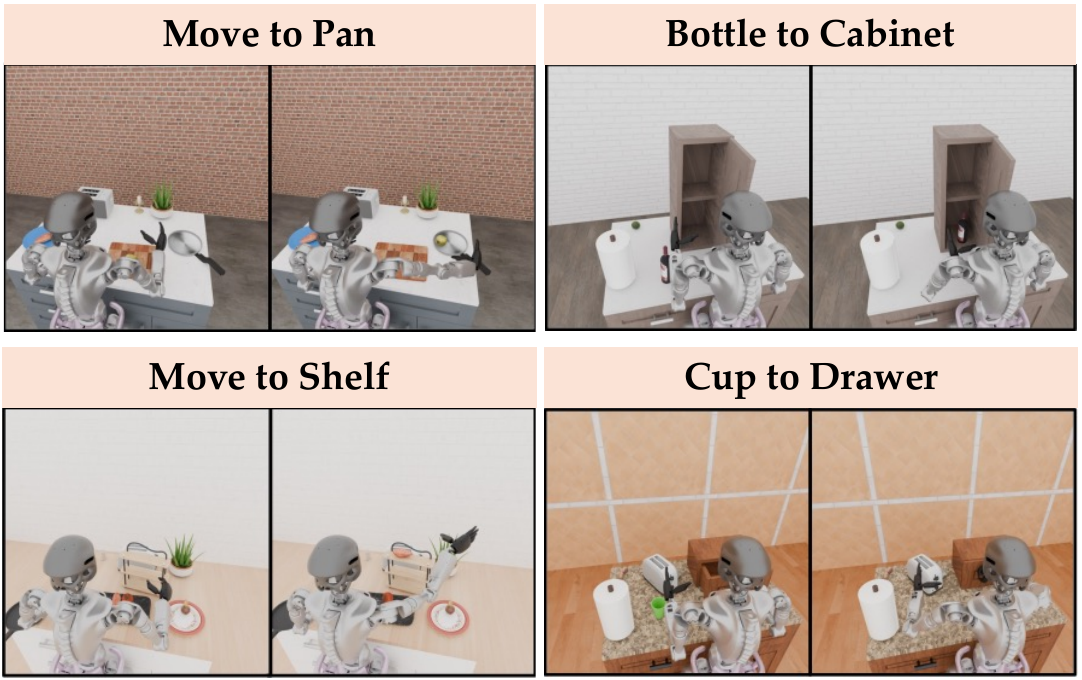}
    \caption{\small
        \textbf{RoboCasa-GR1 tasks.} Some example tasks from the RoboCasa-GR1 tabletop manipulation benchmark. The rendered images of these tasks are taken from \citet{bjorck2025gr00t}.
    }
    \label{fig:robocasa_tasks}
\end{figure}

RoboCasa-GR1~\citep{nasiriany2024robocasa} is a benchmark of tabletop manipulation tasks designed for the GR1 humanoid robot platform. It comprises 24 tasks spanning diverse manipulation skills including picking, placing, pouring, grasping, and insertion. These tasks require controlling a high-DoF humanoid upper body, making them substantially more complex than standard arm-only manipulation benchmarks. The action space includes joint-level control of the torso, arms, and grippers, with observations consisting of joint positions, velocities, and object states.

The 24 tasks fall into two categories. Six base pick-and-place tasks involve manipulating everyday objects into containers: placing a can, cup, or bottle into a drawer; placing a milk carton or potato into a microwave; and placing a bottle or wine bottle into a cabinet. Eighteen novel generalization tasks (\texttt{PosttrainPnPNovel}) evaluate zero-shot transfer to unseen object--container pairs (e.g., cuttingboard-to-basket, tray-to-tiered-shelf, plate-to-cardboard-box). All tasks use the \texttt{GR1ArmsAndWaistFourierHands} embodiment, which provides a 44-dimensional action space across 8 body parts (left/right arm, left/right hand, left/right leg, neck, waist) and a matching 44-dimensional proprioceptive state space. Each task provides 1,000 successful human teleoperation demonstrations.

%% file: sections/appendix_implementation.tex
\section{Implementation Details}
\label{appendix:impl}

\subsection{Computational Resources}

We use NVIDIA RTX 4090 GPUs for all experiments. A single OGBench offline-to-online run takes approximately 5--7 GPU hours, depending on the domain and dataset size. The 100M-transition datasets (\te{cube-quadruple})  require periodic chunking of the offline data to fit in CPU memory, adding roughly 2 hours per run. For the RoboCasa-GR1 experiments, each training run takes around 20 hours due to the high-dimensional physics simulation.

To reproduce the main OGBench results in \cref{tab:main}, we estimate $\underbrace{6}_{\text{hours}} \times \underbrace{2}_{\text{methods}} \times \underbrace{25}_{\text{tasks}} \times \underbrace{4}_{\text{seeds}} \approx 1{,}200$ GPU hours. The Robomimic experiments require $\underbrace{10}_{\text{hours}} \times \underbrace{2}_{\text{methods}} \times \underbrace{3}_{\text{tasks}} \times \underbrace{5}_{\text{seeds}} \approx 300$ GPU hours. The RoboCasa-GR1 experiments take approximately $\underbrace{20}_{\text{hours}} \times \underbrace{4}_{\text{methods}} \approx 80$ GPU hours. In total, reproducing all results in this paper requires roughly 1{,}600 GPU hours.

\subsection{OGBench and Robomimic Experiments}
For tasks that are already evaluated in prior work, we use the reported results. Specifically, results of \te{cube-double}, \te{cube-triple}, \te{cube-quadruple}, \te{scene-sparse}, \te{puzzle-3x3-sparse}, \te{lift}, \te{can}, and \te{square} are taken from \citet{li2025reinforcement}. Otherwise, we implement the baselines in our codebase and evaluated according to the same protocol used in those works.

\hourgreen{\ours{}}.
We sweep over expectile values $\kappa_v \in \{0.5, 0.7, 0.9, 0.93, 0.95, 0.99\}$, horizon lengths $h \in \{5, 10, 25\}$, and scale sets $\Kset \in \{\{1, 5\}, \{1, 5, 10\}, \{1, 5, 10, 25\}\}$. For each environment, we report the best configuration selected in \cref{tab:task-hparams}.

\hourqc{\dqc{}}.
\dqc{} decouples the critic horizon from the policy execution horizon, using a long-horizon critic for credit assignment and a shorter-horizon policy for robust execution. We adapt \dqc{} from its original goal-conditioned formulation~\citep{li2026decoupled} to the reward-based setting by removing goal conditioning entirely and training on the environment's direct reward signal. Concretely, we train five networks:
\begin{enumerate}
    \item $Q_\phi(s, a_1, \ldots, a_h): \mathcal{S} \times \mathcal{A}^h \mapsto \mathbb{R}$ --- the long-horizon (full-chunk) critic. An ensemble of $N_Q = 2$ networks is used.
    \item $Q_\psi(s, a_1, \ldots, a_{h_a}): \mathcal{S} \times \mathcal{A}^{h_a} \mapsto \mathbb{R}$ --- the short-horizon (partial) critic, trained via expectile distillation from $Q_\phi$. An ensemble of $N_Q = 2$ networks is used.
    \item $V_\xi(s): \mathcal{S} \mapsto \mathbb{R}$ --- the value function, trained via expectile regression on the partial critic $Q_\psi$. Provides the TD bootstrap for $Q_\phi$.
    \item $\bar{Q}_\phi$, $\bar{Q}_\psi$, $\bar{V}_\xi$ --- exponential-moving-average target networks for each of the above, updated with decay $\tau = 0.005$.
    \item $f_\omega(s, x, u): \mathcal{S} \times \mathbb{R}^{h_a d_a} \times [0, 1] \mapsto \mathbb{R}^{h_a d_a}$ --- the flow-matching behavior policy, parameterized by a velocity prediction network.
\end{enumerate}
Given a transition $(s_t, a_{t:t+h-1}, s_{t+h}, r_t^{(h)}) \sim \mathcal{D}$ where $r_t^{(h)} = \sum_{j=0}^{h-1} \gamma^j r_{t+j}$ and continuation mask $m_t = \prod_{j=0}^{h-1} \gamma_{\text{cont}}(s_{t+j})$, the training objectives are:

\textbf{(1) Long-horizon critic loss.} The full-chunk critic $Q_\phi$ is trained with an $h$-step TD target bootstrapped from the value function:
\begin{align}
    L(\phi) = \left(Q_\phi(s_t, a_{t:t+h-1}) - r_t^{(h)} - \gamma^h m_t \bar{V}_\xi(s_{t+h})\right)^2.
\end{align}
The value target $\bar{V}_\xi$ (from the EMA network) encodes the long-range return from $s_{t+h}$ onward, avoiding $O(H/h)$ chained 1-step TD updates.

\textbf{(2) Partial critic distillation loss.} The partial critic $Q_\psi$ is trained to approximate the upper envelope of the long-horizon critic $Q_\phi$ via $\kappa_d$-expectile regression:
\begin{align}
    L(\psi) = \mathbb{E}_2^{\kappa_d}\!\left[\bar{Q}_\phi(s_t, a_{t:t+h-1}) - Q_\psi(s_t, a_{t:t+h_a-1})\right],
\end{align}
where $\mathbb{E}_2^{\kappa_d}(x) = |\kappa_d - \mathbb{1}(x < 0)| x^2$ is the asymmetric expectile loss and $\bar{Q}_\phi$ is the EMA target of $Q_\phi$. With $\kappa_d > 0.5$, $Q_\psi$ learns an optimistic approximation of $Q_\phi$ restricted to the first $h_a$ actions.

\textbf{(3) Value implicit backup loss.} The value function $V_\xi$ is trained via $\kappa_b$-expectile regression on the partial critic:
\begin{align}
    L(\xi) = \mathbb{E}_2^{\kappa_b}\!\left[\bar{Q}_\psi(s_t, a_{t:t+h_a-1}) - V_\xi(s_t)\right],
\end{align}
where $\bar{Q}_\psi$ is the EMA target of $Q_\psi$ and $\kappa_b > 0.5$ controls the degree of optimism. This loss is co-trained with the distillation loss in a single backward pass.

\textbf{(4) Flow-matching behavior policy loss.} The behavior policy is trained with standard conditional flow matching:
\begin{align}
    L(\omega) = \left\|f_\omega\!\left(s_t, u a_{t:t+h_a-1} + (1-u)z_t, u\right) - \left(a_{t:t+h_a-1} - z_t\right)\right\|_2^2,
\end{align}
where $u \sim U([0, 1])$ and $z_t \sim \mathcal{N}(0, I_{h_a d_a})$. Note the flow is trained on $h_a$-step chunks, not $h$-step.

\textbf{Policy inference.} At inference time, we sample $N = 32$ candidate $h_a$-step action chunks from the flow-matching policy via Euler integration, score each by the partial critic ensemble $\bar{Q}_\psi$, and select the highest-scoring chunk.

\textbf{Hyperparameter Tuning.} We sweep over backup expectile $\kappa_b \in \{0.5, 0.7, 0.9, 0.93, 0.95, 0.99\}$, distillation expectile $\kappa_d \in \{0.5, 0.8\}$, horizon length $h \in \{5, 10, 25\}$, and policy chunk size $h_a \in \{1, 5, 25\}$. Task-specific hyperparameters are listed in \cref{tab:task-hparams}.

\begin{table}[h]
    \centering
    \begin{tabular}{@{}c c@{}}
        \toprule
        \textbf{Parameter} & \textbf{Value} \\
        \midrule
        Optimizer & AdamW \\
        Learning rate & $3 \times 10^{-4}$ \\
        Batch size & 256 \\
        Discount $\gamma$ & 0.99 \\
        Network width & 512 \\
        Network depth & 4 hidden layers \\
        EMA decay $\tau$ & 0.005 \\
        Ensemble size $N_Q$ & 2 \\
        UTD ratio & 1 \\
        Number of flow steps & 10 \\
        Number of offline training steps & $1{,}000{,}000$ \\
        Number of online environment steps & $1{,}000{,}000$ \\
        \bottomrule
    \end{tabular}
    \vspace{2mm}
    \caption{\textbf{Common hyperparameters for OGBench and Robomimic experiments.} Shared across all tasks unless otherwise noted in \cref{tab:task-hparams}.}
    \label{tab:common-hparams}
\end{table}

\begin{table}[h]
    \centering
    \small
    \begin{tabular}{@{}c c c c c c c c@{}}
        \toprule
        & \multicolumn{3}{c}{\hourgreen{\ours{}}} & \multicolumn{4}{c}{\hourqc{\dqc{}}} \\
        \cmidrule(lr){2-4} \cmidrule(lr){5-8}
        \textbf{Environment} & $\boldsymbol{\kappa_v}$ & $\boldsymbol{h}$ & $\boldsymbol{\Kset}$ & $\boldsymbol{\kappa_b}$ & $\boldsymbol{\kappa_d}$ & $\boldsymbol{h}$ & $\boldsymbol{h_a}$ \\
        \midrule
        \te{cube-double} & 0.9 & 5 & $\{1, 5\}$ & 0.99 & 0.8 & 5 & 1 \\
        \te{cube-triple} & 0.9 & 5 & $\{1, 5\}$ & 0.7 & 0.8 & 10 & 5 \\
        \te{cube-quadruple} & 0.9 & 10 & $\{1, 5, 10\}$ & 0.95 & 0.8 & 10 & 5 \\
        \te{scene-sparse} & 0.95 & 5 & $\{1, 5\}$ & 0.9 & 0.5 & 5 & 1 \\
        \te{puzzle-3x3-sparse} & 0.95 & 5 & $\{1, 5\}$ & 0.9 & 0.8 & 10 & 5 \\
        \te{lift} & 0.93 & 5 & $\{1, 5\}$ & 0.93 & 0.8 & 5 & 1 \\
        \te{can} & 0.9 & 5 & $\{1, 5\}$ & 0.95 & 0.8 & 5 & 1 \\
        \te{square} & 0.9 & 5 & $\{1, 5\}$ & 0.9 & 0.5 & 10 & 5 \\
        \bottomrule
    \end{tabular}
    \vspace{2mm}
    \caption{\textbf{Task-specific hyperparameters for \ours{} and \dqc{}.} Best configuration per environment.}
    \label{tab:task-hparams}
\end{table}

\subsection{Selecting $\mathcal{K}_{\text{set}}$.}
The candidate chunk size set $\Kset$ follows a structured pattern rather than arbitrary search. We define a universal superset $\Kset_{\text{univ}} = \{1, 5, 10, 25\}$ and, for each task, use $\Kset = \{k \in \Kset_{\text{univ}} : k \leq h\}$. The values in $\Kset_{\text{univ}}$ correspond to meaningful timescales in manipulation: $k{=}1$ enables reactive per-step control near contact events, $k{=}5$ covers short primitives (a single reach or grasp), $k{=}10$ spans multi-step subroutines (reach--grasp--lift), and $k{=}25$ captures extended task-level sequences. This nested construction ensures that longer-horizon tasks gain access to coarser timescales for temporal smoothing, while all tasks retain $k{=}1$ as a fallback for contact-rich phases. In practice, the selector rarely needs more than 3--4 candidates: adding intermediate values (e.g., $\Kset = \{1, 2, 3, 4, 5\}$) yields no benefit over $\{1, 5\}$ since the advantage criterion already interpolates between the extremes through the partial critic scores.

\subsection{RoboCasa-GR1 Experiments}

\paragraph{Stage 1: Actor fine-tuning.}
We load the base VLA checkpoint and fine-tune it on the combined 24-task RoboCasa-GR1 dataset (24{,}000 expert demonstrations). During fine-tuning, the diffusion transformer (DiT) action head and visual encoder are trained while the LLM backbone is frozen. We use a learning rate of $3 \times 10^{-5}$, AdamW optimizer, global batch size 32, and 60{,}000 training steps. We then collect 300 rollouts per task using the finetuned policy in randomized environments. We then fine-tune the base policy using behavior cloning (BC) on both expert demonstrations + rollout dataset. This BC model becomes the actor.

\paragraph{Stage 2: Critic training.}
We load the fine-tuned actor checkpoint and attach an \hourgreen{\ours{}} critic head while freezing all backbone parameters. \hourgreen{\ours{}} critic is trained with chunk size $h = 16$ and policy chunk sizes $\Kset = \{1,4,8,16\}$. State and action spaces are normalized using the 99th-percentile statistics from the dataset. Sparse binary rewards (0 on success, $-1$ otherwise) are shaped by setting the final 15 timesteps before success to 1, with a global shift of $-1.0$ applied to all non-success timesteps.

\textbf{QC and DQC baselines.} \hourqc{\qc{}} training follows the same pipeline as \hourgreen{\ours{}} but with a single-scale critic (no policy chunk size selection). \hourqc{\dqc{}} training uses critic chunk size $h = 16$, policy chunk size $k = 4$, discount $\gamma = 0.999$, distillation expectile $\kappa_d = 0.5$, and backup expectile $\kappa_b = 0.9$. All other hyperparameters match \hourgreen{\ours{}}.

\paragraph{Evaluation.}
After critic training, we evaluate using a best-of-$N$ selection policy: the BC actor samples $N = 10$ action candidates, the \hourgreen{\ours{}} critic scores each sample using per-scale advantage, and the $(k^*, a^*)$ pair with the highest score is selected.

\begin{table}[h]
    \centering
\begin{tabular}{@{}l c c c@{}}
        \toprule
        \textbf{Parameter} & \textbf{Stage 1 (Actor)} & \textbf{Stage 2 (Critic)} \\
        \midrule
        Optimizer & AdamW & AdamW \\
        Learning rate & $3 \times 10^{-5}$ & $3 \times 10^{-4}$ \\
        Batch size & 32 & 64 \\
        Training steps & 60{,}000 & 30{,}000 \\
        LR schedule & Cosine & Cosine \\
        Weight decay & $10^{-5}$ & $10^{-5}$ \\
        Trainable params & DiT + visual encoder & Critic head only \\
        Discount $\gamma$ & --- & 0.99 \\
        Chunk sizes & --- & $h=16$, $\Kset=\{1,4,8,16\}$ \\
        Expectile $\kappa_v$ & --- & 0.9 \\
        Best-of-$N$ & --- & 10 \\
        \bottomrule
    \end{tabular}
    \vspace{2mm}
    \caption{\textbf{RoboCasa-GR1 training hyperparameters.} Two-stage pipeline: actor fine-tuning followed by critic training.}
    \label{tab:robocasa-hyperparams}
\end{table}

%% file: sections/full_results.tex
\newpage
\section{Full Results}
\label{app:full_results}

\begin{table}[h]
\centering
\label{tab:full-main-o2o}
\makebox[\textwidth]{\scalebox{0.75}{
\begin{tabular}{lccccccccc}
\toprule
 & & \texttt{RLPD} & \texttt{IQL} & \texttt{FQL} & \texttt{FQL-n} & \texttt{\textbf{QC-FQL}} & \texttt{\textbf{QC}} & \texttt{DQC} & \texttt{AQC} \\
\midrule
 & \texttt{\texttt{task1}} & \cellcolor{cpurple!10}$\phantom{0}\scaleto{\text{--}}{1.7pt}\phantom{0}\rightarrow\overset{\scaleto{[100,100]}{3pt}}{{\color{cpurple}\mathbf{100}}}$ & \cellcolor{cpurple!10}$\overset{\phantom{0}\phantom{0}\scaleto{[0,0]}{3pt}}{\phantom{0}\phantom{0}{\color{gray!60}\mathbf{0}}}\rightarrow\overset{\phantom{0}\scaleto{[97,100]}{3pt}}{\phantom{0}{\color{cpurple}\mathbf{99}}}$ & \cellcolor{cpurple!10}$\overset{\phantom{0}\scaleto{[98,100]}{3pt}}{\phantom{0}{\color{cpurple}\mathbf{99}}}\rightarrow\overset{\scaleto{[100,100]}{3pt}}{{\color{cpurple}\mathbf{100}}}$ & \cellcolor{cpurple!10}$\overset{\scaleto{[98,100]}{3pt}}{{\color{cpurple}\mathbf{100}}}\rightarrow\overset{\scaleto{[100,100]}{3pt}}{{\color{cpurple}\mathbf{100}}}$ & \cellcolor{qccolor!10}$\overset{\scaleto{[100,100]}{3pt}}{{\color{qccolor}\mathbf{100}}}\rightarrow\overset{\scaleto{[100,100]}{3pt}}{{\color{qccolor}\mathbf{100}}}$ & \cellcolor{qccolor!10}$\overset{\scaleto{[100,100]}{3pt}}{{\color{qccolor}\mathbf{100}}}\rightarrow\overset{\scaleto{[100,100]}{3pt}}{{\color{qccolor}\mathbf{100}}}$ & \cellcolor{qccolor!10}$\overset{\phantom{0}\scaleto{[94,100]}{3pt}}{{\phantom{0}\color{qccolor}\mathbf{97}}}\rightarrow\overset{\phantom{0}\scaleto{[98,100]}{3pt}}{{\phantom{0}\color{qccolor}\mathbf{99}}}$ & \cellcolor{ourcolor!10}$\overset{\scaleto{[100,100]}{3pt}}{{\color{ourcolor}\mathbf{100}}}\rightarrow\overset{\scaleto{[100,100]}{3pt}}{{\color{ourcolor}\mathbf{100}}}$ \\
\bottomrule
\end{tabular}
}}
\vspace{2mm}
\caption{\textbf{Full OGBench results table.} 4 seeds, 95\% confidence interval}
\end{table}

\newpage
\section{RoboCasa Results}
\label{app:robocasa}

\begin{table}[h]
\centering
\label{tab:robocasa-main}
\makebox[\textwidth]{\scalebox{0.8}{
\begin{tabular}{lccccc}
\toprule
 & \hourgr{Gr00t N1.6} & \hourfbc{+ Filtered BC} & \hourqc{+ QC} & \hourqc{+ DQC} & \hourgreen{+ AQC} \\
\midrule
\te{PnPBottleToCabinetClose}        & \cellgr{51.5} & \cellfbc{62.5} & \cellqc{64.0} & \cellqc{53.5} & \cellours{66.0} \\
\te{PnPCanToDrawerClose}            & \cellgr{13.0} & \cellfbc{64.5} & \cellqc{65.0} & \cellqc{42.5} & \cellours{68.5} \\
\te{PnPCupToDrawerClose}            & \cellgr{8.5}  & \cellfbc{56.0} & \cellqc{59.0} & \cellqc{54.5} & \cellours{64.0} \\
\te{PnPMilkToMicrowaveClose}        & \cellgr{14.0} & \cellfbc{35.5} & \cellqc{44.5} & \cellqc{21.0} & \cellours{56.0} \\
\te{PnPPotatoToMicrowaveClose}      & \cellgr{41.5} & \cellfbc{51.0} & \cellqc{54.5} & \cellqc{53.0} & \cellours{63.5} \\
\te{PnPWineToCabinetClose}          & \cellgr{16.5} & \cellfbc{34.0} & \cellqc{36.5} & \cellqc{31.0} & \cellours{43.5} \\
\te{PosttrainPnPNovelFromCuttingboardToBasketSplitA}        & \cellgr{58.0} & \cellfbc{60.5} & \cellqc{62.5} & \cellqc{55.0} & \cellours{64.0} \\
\bottomrule
\end{tabular}
}}
\vspace{2mm}
\caption{\textbf{RoboCasa GR1 tabletop results.} Success rate (\%) across 24 tasks on 200 evaluation episodes each.}
\end{table}

\section{Additional Ablation Experiments}
\label{app:ablations}

\begin{figure}[h]
    \begin{minipage}[b]{0.64\linewidth}
        \centering
        \includegraphics[width=\linewidth]{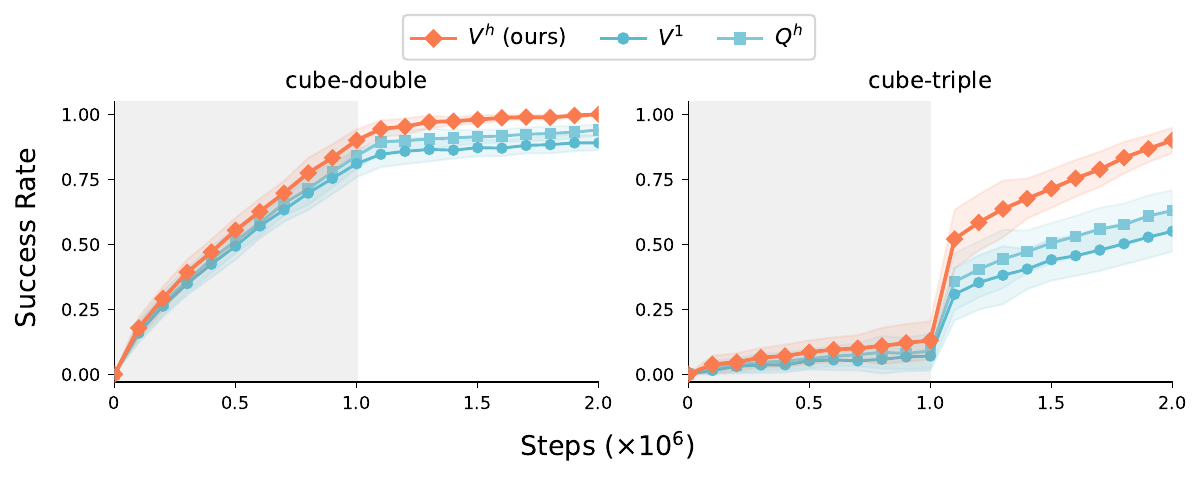}
        \caption{
            \textbf{Bootstrap quality ablation.}
            Comparing $V^h$, $V^1$, and $Q^h$ as bootstrap sources for $Q^k$.
        }
        \label{fig:ablation-bootstrap}
    \end{minipage}
    \hfill
    \begin{minipage}[b]{0.32\linewidth}
        \centering
        \includegraphics[width=\linewidth]{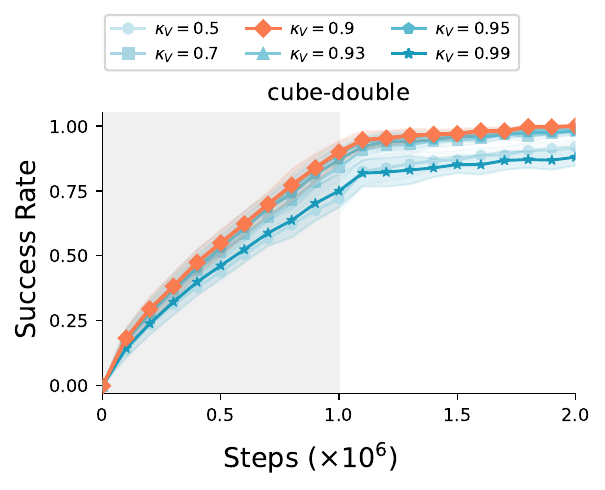}
        \caption{
            \textbf{$\kappa_V$ sensitivity on \te{cube-double}.}
        }
        \label{fig:ablation-kv}
    \end{minipage}
\end{figure}

\paragraph{Bootstrap quality.}
\cref{fig:ablation-bootstrap} compares three bootstrap sources for $Q^k$ training: the default long-horizon baseline $V^h$, the 1-step value $V^1$, and using $Q^h$ directly without baseline subtraction.
On \te{cube-double}, $V^h$ yields 90\% offline and 100\% online; $V^1$ underestimates long-horizon returns, dropping to 81\% offline (89\% online).
On \te{cube-triple}, the gap widens: $V^1$ falls to 7\% offline and 55\% online vs.\ 13\% and 90\% for $V^h$, confirming that horizon-aligned bootstrap is critical on longer-horizon tasks where the 1-step value significantly underestimates returns.
Using $Q^h$ directly (no baseline subtraction) falls between the two (84\% offline on cube-double, 9\% on cube-triple), suggesting that baseline subtraction removes state-dependent noise that $Q^h$ alone cannot cancel.

\paragraph{$\kappa_V$ sensitivity.}
\cref{fig:ablation-kv} sweeps the expectile parameter $\kappa_V$ for the $V^h$ and $V^k$ baselines on \te{cube-double}.
The default $\kappa_V = 0.9$ (upper envelope) achieves the highest offline performance (90\% offline, 100\% online).
Values in the range 0.7--0.93 are robust, with offline performance within 6\% of optimal; $\kappa_V = 0.5$ (median, approximating behavior policy value) degrades sharply to 72\% offline, confirming that a conservative baseline is essential.
Very high values ($\kappa_V = 0.99$) also degrade performance (75\%) due to unstable $Q^k$ estimates from near-max baselines, indicating that the upper envelope is a sweet spot between conservatism and stability.

\begin{figure}[h]
    \centering
    \includegraphics[width=\linewidth]{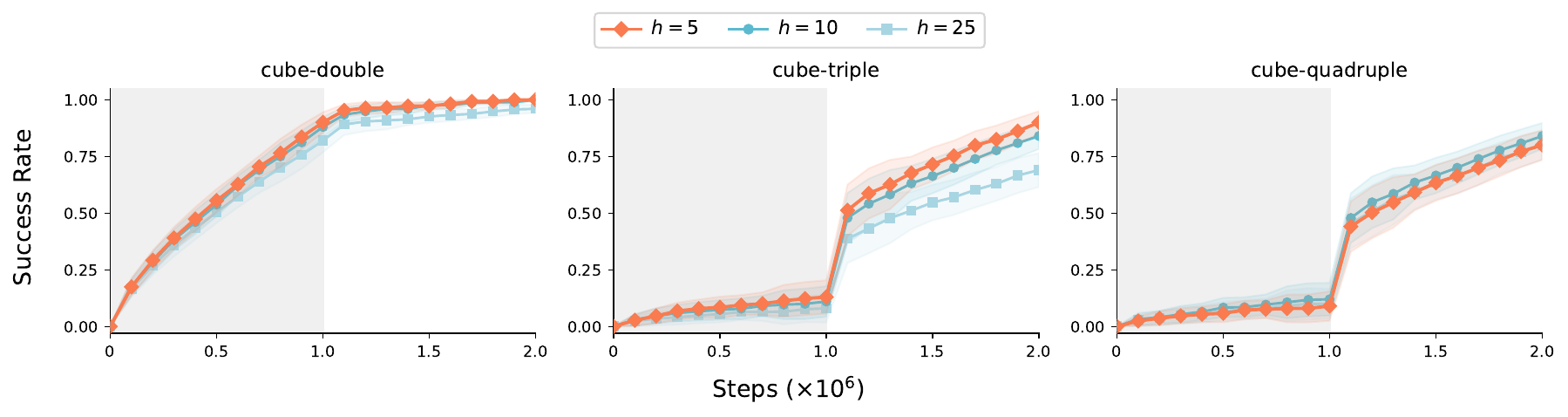}
    \caption{
        \textbf{Critic chunk size $h$ sensitivity.}
        Comparing $h \in \{5, 10, 25\}$ on \te{cube-double}, \te{cube-triple}, and \te{cube-quadruple}.
    }
    \label{fig:ablation-h}
\end{figure}

\paragraph{Critic chunk size $h$ sensitivity.}
\cref{fig:ablation-h} evaluates how \hourgreen{\ours{}} performance varies with the critic chunk size $h$ across \te{cube-double}, \te{cube-triple}, and \te{cube-quadruple}.
On \te{cube-double}, $h = 5$ is optimal (90\% offline, 100\% online); increasing $h$ to 10 or 25 degrades offline performance (88\% and 82\%) without online benefit, suggesting diminishing returns on shorter-horizon tasks.
On \te{cube-triple}, $h = 5$ remains the best choice (13\% offline, 90\% online), with larger $h$ providing no advantage.
On \te{cube-quadruple}, however, $h = 10$ outperforms both $h = 5$ and $h = 25$ (12\% offline, 84\% online vs.\ 9\% and 10\% offline), suggesting that longer credit assignment horizons benefit from larger $h$, but only up to a point beyond which the added variance dominates.

\paragraph{Z-score normalization.}

\begin{figure}[h]
    \centering
    \includegraphics[width=0.75\linewidth]{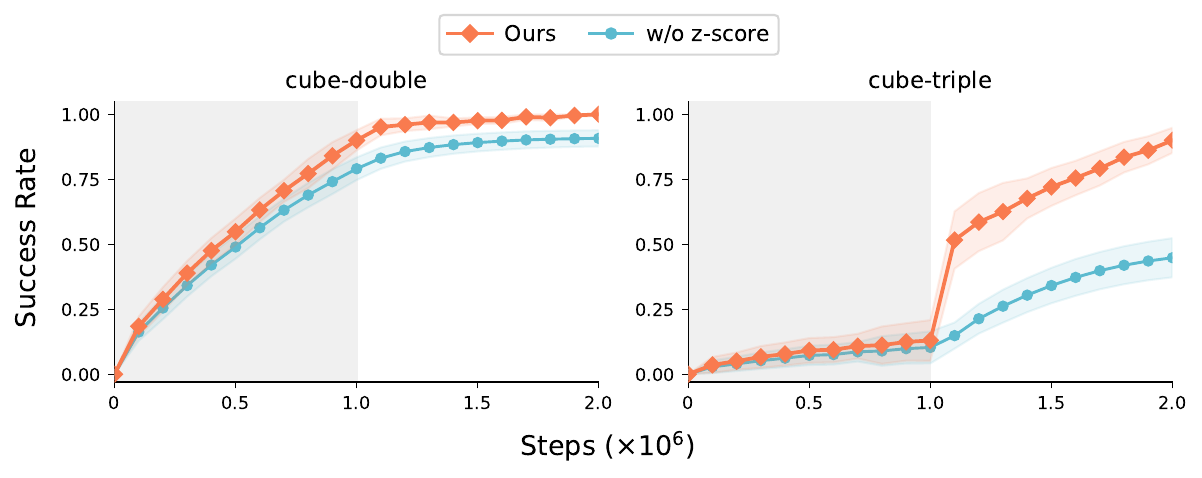}
    \caption{\small
        \textbf{Z-score normalization ablation.}
        Removing per-scale z-score normalization degrades performance across both tasks, with the gap widening on longer-horizon domains (\te{cube-triple}).
    }
    \label{fig:ablation-zscore}
\end{figure}

\cref{fig:ablation-zscore} compares full \hourgreen{\ours{}} against the variant without z-score normalization. On \te{cube-double}, removing z-score normalization reduces final success rate by 9\% (91\% vs 100\%), as the variance mismatch between $k=1$ and $k=5$ advantage scores causes the argmax to be dominated by the higher-variance scale. The effect is more pronounced on \te{cube-triple}: without z-score, performance drops by 45\% (45\% vs 90\%), because the wider range of advantage magnitudes across scales amplifies the dominance problem. This confirms that z-score normalization is necessary to equalize the contribution of each scale to the selection decision, particularly when the candidate set spans multiple orders of magnitude in horizon length.

\paragraph{$\Kset$ robustness.}
The candidate set $\Kset$ is constructed as $\{k \in \Kset_{\mathrm{univ}} : k \leq h\}$ where $\Kset_{\mathrm{univ}} = \{1, 5, 10, 25\}$. We verify that adding intermediate values provides no benefit: on \te{cube-double}, replacing $\Kset = \{1, 5\}$ with $\{1, 2, 3, 4, 5\}$ yields 90\% offline and 99\% online (vs.\ 90\% and 100\% for $\{1, 5\}$), while $\Kset = \{1, 3, 5\}$ gives 89\% offline and 99\% online. This confirms that the advantage criterion already interpolates between the extreme scales through the partial critic scores, and the selector only benefits from having endpoints that bracket the relevant timescales.

\paragraph{Resource efficiency of AQC.}
We compare the parameter counts of all Q-chunking methods on the representative domain \texttt{cube-triple-*} (assuming $h=5$ for all methods). \hourgreen{\ours{}} is the most parameter-heavy at $\approx 7.9$M due to per-scale critics. \hourqc{\dqc{}} uses a single partial critic and a smaller flow policy ($h_a < h$), totaling $\approx6.6$M. \hourqc{\qc{}} and \hourbase{FQL} use a single critic backbone ($\approx 4.2$M and $\approx 4.9$M respectively). \hourbase{RLPD} uses $K=10$ critic networks, making it the largest at $\approx 17.2$M.
\begin{table}[h]
    \centering
    \begin{tabular}{@{}cc@{}}
    \toprule
            Methods & Parameter Count (in millions) \\
    \midrule
    AQC & $\approx 7.9$ \\
    DQC & $\approx 6.6$ \\
    QC & $\approx 4.2$ \\
    QC-FQL & $\approx 5.0$ \\
    FQL & $\approx 4.9$ \\
    RLPD & $\approx 17.2$ \\
    \bottomrule
    \end{tabular}
    \vspace{2mm}
\caption{\small \textbf{Parameter count for each method.} AQC and DQC share the same backbone architecture as QC, adding only lightweight per-scale critic heads.}
    \label{tab:pcnt}
\end{table}

%% file: sections/theory.tex
\section{Theoretical Analysis}
\label{app:theory}

In this section, we develop a comprehensive theoretical framework for \ours{}.
While prior work on action chunking establishes when fixed-chunk Q-learning succeeds under open-loop consistency~\citep{li2026decoupled},
our analysis must address a fundamentally different question: \emph{when does adaptive chunk-size selection succeed, and how much better is it than any fixed chunk size?}
We proceed in stages:
(1) we define adaptive open-loop consistency, the data condition under which multi-scale re-querying is valid;
(2) we prove that the advantage selector correctly identifies the best chunk size under a separability condition;
(3) we establish that \ours{}'s value strictly dominates any fixed-chunk policy;
(4) we bound the regret introduced by imperfect critics;
(5) we analyze the closed-loop execution of the adaptive policy;
(6) we show that the multi-scale bootstrap creates a regularization hierarchy across critic scales.
All proofs appear in Appendix \ref{app:theory-proofs}.

\subsection{Assumptions and notation}
\label{app:theory-assumptions}

We adopt the standard assumption that the data distribution obeys the MDP transition dynamics~\citep{li2026decoupled}:

\begin{assumption}[Data Obeys the Transition Dynamics]
\label{assumption:data}
$\mathcal{D}$ is a trajectory distribution generated by rolling out a (possibly non-Markovian) behavior policy $\pi_\beta$ from an initial state distribution $\mu$.
Each subsequent state is generated according to the MDP dynamics $T$: $s_{t+k+1} \sim T(\cdot \mid s_{t+k}, a_{t+k})$ for all $k \in \{0, 1, \ldots, h-1\}$.
The resulting trajectory is $\{s_t, s_{t+1}, \ldots, s_{t+h}, a_t, a_{t+1}, \ldots, a_{t+h}\} \in \mathcal{T} = \mathcal{S}^{h} \times \mathcal{A}^h$.
\end{assumption}

We write $H = 1/(1-\gamma)$ and $\bar{H}_k = 1/(1-\gamma^k)$ for the 1-step and $k$-step effective horizons, respectively.
We use $V^{\mathrm{AQC}}$ to denote the true closed-loop value of the adaptive policy, $V^{\dagger}$ for the oracle selector's value, $V^{\bullet}$ for the fully reactive closed-loop execution, and $V^k$ (with a fixed $k \in \Kset$) for the fixed-chunk policy value.
The optimal value is $V^\star$.

\subsection{Adaptive Open-Loop Consistency (AOLC)}
\label{app:theory-aolc}

Prior work introduces the open-loop consistency (OLC) condition to characterize when replaying a fixed-length action chunk open-loop produces trajectories close to the data distribution~\citep{li2026decoupled}.
For \ours{}, the situation is more subtle: the agent executes $k^*(s_t)$ steps open-loop and then re-queries, where $k^*$ varies by state.
The re-querying points form a \emph{randomly spaced} set of decision points, and we need the data to cover trajectories that match this adaptive pattern.

\begin{definition}[Adaptive Open-Loop Consistency]
\label{def:aolc}
Let $\kappa: \mathcal{S} \to \Kset$ be a chunk-size selection function.
For a data distribution $\mathcal{D}$ satisfying Assumption \ref{assumption:data},
$\mathcal{D}$ is \textbf{$\varepsilon_\Kset$-adaptively open-loop consistent} under $\kappa$
if for every $s_t \in \mathrm{supp}(P_\mathcal{D}(s_t))$,
\begin{align}
    &D_{\mathrm{TV}}\Bigl(P_{\mathcal{D}}(s_{t+\kappa(s_t)}, a_{t+\kappa(s_t)} \mid s_t)\;\big\|\;
    P^{\circ}_{\mathcal{D},\kappa}(s_{t+\kappa(s_t)}, a_{t+\kappa(s_t)} \mid s_t)\Bigr) \leq \varepsilon_\Kset,
    \label{eq:aolc-sa} \\
    &D_{\mathrm{TV}}\Bigl(P_{\mathcal{D}}(s_{t+\kappa(s_t)} \mid s_t)\;\big\|\;
    P^{\circ}_{\mathcal{D},\kappa}(s_{t+\kappa(s_t)} \mid s_t)\Bigr) \leq \varepsilon_\Kset,
    \label{eq:aolc-s}
\end{align}
where $P^{\circ}_{\mathcal{D},\kappa}$ is the trajectory distribution obtained by rolling out
the marginal policy $\pi^\circ_\mathcal{D}(a_{t:t+k} \mid s_t) := P_\mathcal{D}(a_{t:t+k} \mid s_t)$
with chunk size $\kappa(s_t)$ at each re-query point.
\end{definition}

Intuitively, $\varepsilon_\Kset$ measures how well the data distribution is preserved when actions from $\mathcal{D}$ are replayed open-loop using the \emph{adaptive} chunking schedule $\kappa$.
When $\varepsilon_\Kset = 0$, the data perfectly matches the open-loop rollouts under the adaptive schedule.
The key difference from standard OLC is that the re-querying points are state-dependent, so the total variation distance bound must hold uniformly over the distribution of re-querying times induced by $\kappa$.

\begin{proposition}[AOLC implies per-scale OLC]
\label{prop:aolc-implies-olc}
If $\mathcal{D}$ is $\varepsilon_\Kset$-adaptively open-loop consistent under $\kappa(s) \equiv k$ (constant selection), then $\mathcal{D}$ is $\varepsilon_\Kset$-open-loop consistent for chunk size $k$ in the sense of \citet{li2026decoupled}.
\end{proposition}

This shows that AOLC is a strict generalization: it reduces to standard OLC when the selection function is constant, and imposes a stronger uniform condition when $\kappa$ varies across states.

\subsection{Advantage Separability and Selector Soundness}
\label{app:theory-separability}

The core of \ours{} is the adaptive selector
$k^*(s) \in \argmax_{k \in \Kset} \max_{\ac{t}{t+k}} A^k(s, \ac{t}{t+k})$, where
$A^k(s, \ac{t}{t+k}) := (Q^k(s, \ac{t}{t+k}) - V^k(s))/\gamma^k$ is the per-scale advantage.
We now formalize when this selector is \emph{sound} --- i.e., when it identifies the truly best chunk size.

\begin{definition}[Advantage Separability]
\label{def:as}
Let $A^{k,*}(s, \ac{t}{t+k}) := (Q^{k,*}(s, \ac{t}{t+k}) - V^{k,*}(s))/\gamma^k$ denote the
optimal per-scale advantage for chunk size $k$.
Define the oracle-best chunk size at state $s$ as
\begin{align}
    \kdag(s) \in \argmax_{k \in \Kset}\; \max_{\ac{t}{t+k}} A^{k,*}(s, \ac{t}{t+k}).
\end{align}
The data $\mathcal{D}$ exhibits \textbf{$\Delta$-advantage separability} at $s$ if
\begin{align}
    \max_{\ac{t}{t+\kdag}} A^{\kdag,*}(s, \ac{t}{t+\kdag})
    \;-\;
    \max_{k \neq \kdag}\max_{\ac{t}{t+k}} A^{k,*}(s, \ac{t}{t+k})
    \;\geq\; \Delta(s).
    \label{eq:adv-sep}
\end{align}
We say $\mathcal{D}$ is globally $\Delta$-advantage separable if $\Delta(s) \geq \Delta > 0$ for all $s \in \mathrm{supp}(P_{\mathcal{D}^\star}(s))$.
\end{definition}

Advantage separability measures the \emph{intrinsic discriminability} of the best chunk size.
When $\Delta(s)$ is large, the optimal scale is clearly distinguishable; when $\Delta(s) \approx 0$, multiple scales are nearly equally good and no selector can do much better than random.
In practice, $\Delta(s)$ is large near contact events (where short chunks are uniquely valuable) and in free-space motion (where long chunks clearly dominate), and small in intermediate states.

\begin{theorem}[Selector Soundness]
\label{thm:selector-sound}
Let $\varepsilon_k := \|Q^k - Q^{k,*}\|_\infty$ and $\delta_k := \|V^k - V^{k,*}\|_\infty$ be the critic estimation errors, and let $\bar\varepsilon := \max_{k \in \Kset} (\varepsilon_k + \delta_k)$.
If $\mathcal{D}$ is $\Delta$-advantage separable and $\bar\varepsilon < \Delta \gamma^{k_{\min}} / 2$, where $k_{\min} := \min_{k \in \Kset} k$, then for all $s \in \mathrm{supp}(P_{\mathcal{D}^\star}(s))$, the empirical selector
\begin{align}
    \hat{k}(s) \in \argmax_{k \in \Kset}\; \max_{\ac{t}{t+k}} \frac{Q^k(s, \ac{t}{t+k}) - V^k(s)}{\gamma^k}
\end{align}
agrees with the oracle selector: $\hat{k}(s) = \kdag(s)$.
\end{theorem}

\begin{proof}[Proof sketch]
Let $f_k(s) := \max_{\ac{t}{t+k}} A^{k,*}(s, \ac{t}{t+k})$ and $\hat{f}_k(s) := \max_{\ac{t}{t+k}} \hat{A}^k(s, \ac{t}{t+k})$.
Since $|Q^k - Q^{k,*}| \leq \varepsilon_k$ and $|V^k - V^{k,*}| \leq \delta_k$, we have $|\hat{f}_k(s) - f_k(s)| \leq (\varepsilon_k + \delta_k)/\gamma^k \leq \bar\varepsilon/\gamma^{k_{\min}}$.
Advantage separability gives $f_{\kdag}(s) - \max_{k \neq \kdag} f_k(s) \geq \Delta(s)$.
By the triangle inequality, if $\bar\varepsilon/\gamma^{k_{\min}} < \Delta(s)/2$, then $\hat{f}_{\kdag}(s) > \hat{f}_k(s)$ for all $k \neq \kdag$, so $\hat{k}(s) = \kdag(s)$.
A full proof appears in Appendix \ref{proof:selector-sound}.
\end{proof}

This theorem establishes that the advantage selector is \emph{consistent}: as critic accuracy improves, the selected chunk size converges to the oracle choice.
The bound is non-vacuous precisely when $\Delta(s) > 0$, i.e., when the problem is intrinsically discriminable.

\begin{proposition}[Selector Regret]
\label{prop:selector-regret}
Under the conditions of \cref{thm:selector-sound}, the per-state regret of the empirical selector relative to the oracle is bounded by
\begin{align}
    \left|\hat{A}^{\hat{k}(s)}(s) - A^{\kdag,*}(s)\right|
    \;\leq\;
    \bar\varepsilon + \frac{2\bar\varepsilon}{\gamma^{k_{\min}}\Delta(s)} \cdot \max_{k \in \Kset} \diamA^k,
\end{align}
where $\diamA^k := \sup_{s,a} A^{k,*}(s,a) - \inf_{s,a} A^{k,*}(s,a)$ is the advantage range for scale $k$.
\end{proposition}

\subsection{AQC Value Dominance}
\label{app:theory-dominance}

We now prove the central theoretical claim: \ours{}'s value strictly dominates any fixed-chunk policy.
This requires analyzing how the selector's decisions compound over time.

\begin{definition}[AQC Effective Policy]
\label{def:aqc-policy}
The effective policy of \ours{} is the closed-loop policy $\pi^{\mathrm{AQC}}$ that, at each state $s_t$, selects $\hat{k}(s_t)$ via the advantage criterion and executes $a^*_{1:\hat{k}(s_t)}$ open-loop.
The value function of $\pi^{\mathrm{AQC}}$ is
\begin{align}
    \vaqc(s) := \mathbb{E}_{\pi^{\mathrm{AQC}}}\!\left[\sum_{\tau=0}^\infty \gamma^\tau r(s_\tau, a_\tau) \;\middle|\; s_0 = s\right].
\end{align}
\end{definition}

\begin{theorem}[AQC Dominates Fixed-Chunk Policies]
\label{thm:dominance}
Let $V^k$ denote the true closed-loop value of the fixed-$k$ policy $\pi^k(s) := \argmax_{\ac{t}{t+k}} Q^{k,*}(s, \ac{t}{t+k})$ with open-loop execution.
Assume $\mathcal{D}$ is $\Delta$-advantage separable with $\Delta > 0$ and $\bar\varepsilon < \Delta \gamma^{k_{\min}} / 2$, where $k_{\min} := \min_{k \in \Kset} k$.
Then for all $s \in \mathrm{supp}(P_{\mathcal{D}^\star}(s))$ and any fixed $k \in \Kset$,
\begin{align}
    \vaqc(s) - V^k(s)
    \;\geq\;
    \frac{\gamma^{k_{\min}}(1 - 2\bar\varepsilon/(\gamma^{k_{\min}}\Delta))}{1 - \gamma}\;
    \mathbb{E}_{s' \sim d^{\mathrm{AQC}}}\!\left[\bar{A}^{\kdag,*}(s') - \bar{A}^{k,*}(s')\right],
    \label{eq:dominance-bound}
\end{align}
where $d^{\mathrm{AQC}}$ is the discounted state visitation distribution of $\pi^{\mathrm{AQC}}$ and $\bar{A}^{j,*}(s) := \max_{\ac{t}{t+j}} A^{j,*}(s, \ac{t}{t+j})$ is the best advantage at scale $j$.
\end{theorem}

\begin{proof}[Proof sketch]
The proof adapts the performance difference lemma of \citet{kakade2002approximately} to the chunk-selection setting via a meta-MDP construction:

1. \textbf{Meta-MDP}: We define a meta-MDP whose ``actions'' are pairs $(k, \ac{t}{t+k})$ and whose transitions correspond to executing $k$ steps open-loop. In this meta-MDP, the performance difference lemma gives:
\begin{align}
    \vaqc(s) - V^k(s)
    = \frac{1}{1-\gamma} \mathbb{E}_{s' \sim d^{\mathrm{AQC}}}\!\left[
    \gamma^{k^*(s')} \bar{A}^{k^*(s'),*}_{\mathrm{meta}}(s')
    - \gamma^k \bar{A}^{k,*}_{\mathrm{meta}}(s')
    \right].
\end{align}

2. \textbf{Selector correctness}: Under advantage separability with $\bar\varepsilon < \Delta \gamma^{k_{\min}}/2$, the selector always chooses $k^*(s') = \kdag(s')$, so $\bar{A}^{k^*(s'),*}(s') \geq \bar{A}^{k,*}(s')$ for any fixed $k$.

3. \textbf{Discount normalization}: Since $\gamma^{k^*(s')} \geq \gamma^{k_{\min}}$, we lower-bound $\gamma^{k^*(s')} \bar{A}^{\kdag,*}(s') - \gamma^k \bar{A}^{k,*}(s') \geq \gamma^{k_{\min}}(\bar{A}^{\kdag,*}(s') - \bar{A}^{k,*}(s'))$.

4. \textbf{Error accounting}: The $(1 - 2\bar\varepsilon/(\gamma^{k_{\min}}\Delta))$ factor accounts for the probability of mis-selection.

A full proof appears in Appendix \ref{proof:dominance}.
\end{proof}

The bound has a clean interpretation: the dominance gap is proportional to (a) how often the selector is correct ($1 - 2\bar\varepsilon/(\gamma^{k_{\min}}\Delta)$), (b) the effective horizon $1/(1-\gamma)$, and (c) the average advantage gap between the oracle and the fixed-$k$ policy.
When there exists a set of states where scale $k$ is suboptimal, the bound is strictly positive.

\begin{corollary}[Strict Dominance]
\label{cor:strict-dominance}
If there exists a set of states $\mathcal{S}_{\mathrm{diff}} \subseteq \mathrm{supp}(P_{\mathcal{D}^\star})$ with $d^{\mathrm{AQC}}(\mathcal{S}_{\mathrm{diff}}) > 0$ such that for each $s \in \mathcal{S}_{\mathrm{diff}}$, there exist $k_1, k_2 \in \Kset$ with $\bar{A}^{k_1,*}(s) \neq \bar{A}^{k_2,*}(s)$, then $\vaqc(s) > V^k(s)$ for all $k \in \Kset$ whenever $\bar\varepsilon < \Delta_{\min}/2$, where $\Delta_{\min} := \min_{s \in \mathcal{S}_{\mathrm{diff}}} \Delta(s)$.
\end{corollary}

\subsection{Error Propagation Under Imperfect Critics}
\label{app:theory-error-prop}

When the critics $\{Q^k, V^k\}$ are learned from finite data, the selector's decisions are based on noisy advantage estimates.
We now analyze how this noise propagates through the closed-loop execution.

\begin{definition}[Critic Error Profile]
\label{def:error-profile}
For each $k \in \Kset$, define the critic errors $\varepsilon_k := \sup_{s, \ac{t}{t+k}} |Q^k(s, \ac{t}{t+k}) - Q^{k,*}(s, \ac{t}{t+k})|$ and $\delta_k := \sup_s |V^k(s) - V^{k,*}(s)|$.
The composite error is $\bar\varepsilon_{\Kset} := \max_{k \in \Kset} (\varepsilon_k + \delta_k)$.
\end{definition}

\begin{theorem}[Adaptive Selector Regret]
\label{thm:selector-regret}
Let $\pi^{\mathrm{AQC}}$ be the policy using the learned selector $\hat{k}(s)$ and $\pi^{\dagger}$ be the oracle policy using $\kdag(s)$.
Then for all $s \in \mathrm{supp}(P_{\mathcal{D}^\star}(s))$,
\begin{align}
    \left|\vaqc(s) - \vdag(s)\right|
    \;\leq\;
    \frac{2\bar\varepsilon_{\Kset}}{(1 - \gamma)\gamma^{k_{\min}}\Delta(s)} \cdot
    \max_{k \in \Kset} \frac{\Rbar_k}{\gamma^k},
    \label{eq:regret-bound}
\end{align}
where $\Rbar_k := \sup_{s, \ac{t}{t+k}} |Q^{k,*}(s, \ac{t}{t+k}) - V^{k,*}(s)|$ is the maximum advantage magnitude for scale $k$.
\end{theorem}

\begin{proof}[Proof sketch]
The proof combines the selector soundness bound (\cref{thm:selector-sound}) with a value difference analysis:

1. At each re-query state $s$, the selector incurs mis-selection probability at most $2\bar\varepsilon_{\Kset}/(\gamma^{k_{\min}}\Delta(s))$.

2. When a mis-selection occurs, the value loss per mistake is bounded by $\Rbar_k/\gamma^k$.

3. Over the trajectory, the expected number of mis-selections in the discounted horizon is at most $(1/(1-\gamma)) \cdot 2\bar\varepsilon_{\Kset}/(\gamma^{k_{\min}}\Delta(s))$.

4. The total regret is the product of expected mis-selection count and per-mis-selection loss.

A full proof appears in Appendix \ref{proof:selector-regret}.
\end{proof}

\begin{corollary}[Uniform Regret Bound]
\label{cor:uniform-regret}
If $\Delta(s) \geq \Delta > 0$ for all $s$, then
\begin{align}
    \|\vaqc - \vdag\|_\infty
    \;\leq\;
    \frac{2\bar\varepsilon_{\Kset}}{(1 - \gamma)\gamma^{k_{\min}}\Delta} \cdot
    \max_{k \in \Kset} \frac{\Rbar_k}{\gamma^k}.
\end{align}
\end{corollary}

This bound shows that the regret scales linearly with the critic error and inversely with the advantage gap.
The $1/(1-\gamma)$ factor reflects the compounding of selector mistakes over the effective horizon.

\subsection{Closed-Loop Analysis for Adaptive Execution}
\label{app:theory-closed-loop}

\ours{} executes $\hat{k}(s_t)$ steps open-loop and then re-queries.
This adaptive re-querying pattern is fundamentally different from both DQC's fixed-chunk execution and standard 1-step closed-loop control.
We now analyze the optimality gap of the adaptive closed-loop policy.

\begin{definition}[Adaptive Closed-Loop Value]
\label{def:adaptive-closed}
Let $\pi^{\mathrm{AQC}}_{\mathrm{CL}}$ execute the learned policy in closed-loop: at each step, select $\hat{k}(s)$, sample $a^*_{1:\hat{k}(s)}$, execute $a^*_1$, and re-query at $s_{t+1}$.
The closed-loop value is
\begin{align}
    \vcl(s) := \mathbb{E}_{\pi^{\mathrm{AQC}}_{\mathrm{CL}}}
    \left[\sum_{\tau=0}^\infty \gamma^\tau r(s_\tau, a_\tau) \;\middle|\; s_0 = s\right].
\end{align}
\end{definition}

\begin{theorem}[Adaptive Closed-Loop Optimality]
\label{thm:closed-loop}
Assume $\mathcal{D}$ is $\varepsilon_\Kset$-adaptively open-loop consistent under the oracle selector $\kdag$, viewed as a fixed selection function $\kappa(s) := \kdag(s)$ defined by the advantage separability condition.
Then for all $s \in \mathrm{supp}(P_{\mathcal{D}^\star}(s))$,
\begin{align}
    V^\star(s) - \vdag_{\mathrm{CL}}(s)
    \;\leq\;
    \frac{\varepsilon_\Kset \gamma}{(1-\gamma)^2}
    \left[\frac{2}{1 - (1 - 2\varepsilon_\Kset)\gamma^{k_{\min}}}
    + \frac{1}{1 - (1 - \varepsilon_\Kset)\gamma^{k_{\min}}}\right],
    \label{eq:closed-loop-bound}
\end{align}
where $k_{\min} := \min_{k \in \Kset} k$.
Furthermore, when the learned selector is used,
\begin{align}
    V^\star(s) - \vcl(s)
    \;\leq\;
    (V^\star(s) - \vdag_{\mathrm{CL}}(s))
    + \frac{2\bar\varepsilon_{\Kset}}{(1-\gamma)\gamma^{k_{\min}}\Delta} \cdot
    \max_{k} \frac{\Rbar_k}{\gamma^k}.
    \label{eq:closed-loop-learned}
\end{align}
\end{theorem}

\begin{proof}[Proof sketch]
The proof combines three ingredients:

1. \textbf{AOLC decomposition}: At each re-query point, the state distribution deviates from the data by at most $\varepsilon_\Kset$ in TV distance. Using standard total-variation-to-value bounds, the per-re-query value error is bounded by $\varepsilon_\Kset$ times the value range.

2. \textbf{Geometric compounding}: The number of re-queries in the discounted horizon is bounded by $1/(1 - \gamma^{k_{\min}})$, since at minimum $k_{\min}$ steps elapse between re-queries. This gives the effective horizon factor.

3. \textbf{Selector error addition}: The learned selector introduces additional regret bounded by \cref{thm:selector-regret}.

The bound is structurally similar to DQC's closed-loop bound but with $k_{\min}$ replacing $h$, reflecting that the adaptive policy re-queries more frequently (at worst every $k_{\min}$ steps), which improves reactivity.
A full proof appears in Appendix \ref{proof:closed-loop}.
\end{proof}

The bound reveals a trade-off: a larger $\Kset$ with small $k_{\min}$ enables better reactivity near contacts (since the policy can re-query sooner), but increases the number of re-query points where TV distance error can accumulate.

\subsection{Multi-Scale Bootstrap Regularization}
\label{app:theory-bootstrap}

AQC's unique architectural choice is that all partial critics $\{Q^k\}_{k < h}$ bootstrap from the same long-horizon value function $V^h$.
We now prove that this creates a \emph{value flow hierarchy} that implicitly regularizes the critic ensemble.

\begin{theorem}[Multi-Scale Bootstrap Regularization]
\label{thm:bootstrap}
Let $\varepsilon_h := \|V^h - V^{h,*}\|_\infty$ and let $\varepsilon_k$ be the $k$-step TD fitting error for $Q^k$.
Then for all $k \in \Kset$,
\begin{align}
    \|Q^k - Q^{k,*}\|_\infty
    \;\leq\;
    \underbrace{\frac{\gamma^k}{1 - \gamma^k} \varepsilon_h}_{\text{bootstrap error}}
    \;+\;
    \underbrace{\frac{1}{1 - \gamma^k} \varepsilon_k}_{\text{fitting error}}.
    \label{eq:bootstrap-bound}
\end{align}
Moreover, the total critic ensemble error is bounded by
\begin{align}
    \sum_{k \in \Kset} \|Q^k - Q^{k,*}\|_\infty
    \;\leq\;
    \varepsilon_h \sum_{k \in \Kset} \frac{\gamma^k}{1 - \gamma^k}
    \;+\; \sum_{k \in \Kset} \frac{\varepsilon_k}{1 - \gamma^k}.
    \label{eq:ensemble-bound}
\end{align}
\end{theorem}

\begin{proof}[Proof sketch]
The Bellman operator $\mathcal{T}^k_h Q = R_{t:t+k} + \gamma^k V^h$ is a $\gamma^k$-contraction in the sup-norm.
Let $Q^{k}_{\mathrm{fp}}$ be its fixed point. Then $\|Q^{k}_{\mathrm{fp}} - Q^{k,*}\|_\infty \leq \gamma^k \|V^h - V^{h,*}\|_\infty = \gamma^k \varepsilon_h$.
Since the TD loss drives $Q^k$ toward $Q^{k}_{\mathrm{fp}}$ with fitting error $\varepsilon_k$, the contraction mapping theorem gives $\|Q^k - Q^{k}_{\mathrm{fp}}\|_\infty \leq \varepsilon_k / (1 - \gamma^k)$.
Triangle inequality yields the per-scale bound; summing over $k$ gives the ensemble bound.
A full proof appears in Appendix \ref{proof:bootstrap}.
\end{proof}

The key insight is that a single $\varepsilon_h$ term controls the bootstrap error for \emph{all} partial critics simultaneously, creating a hierarchical regularization: improving $V^h$ automatically improves all $Q^k$ critics.

\begin{corollary}[Bootstrap Advantage over 1-Step]
\label{cor:bootstrap-advantage}
If $Q^k$ were bootstrapped from $V^1$ instead of $V^h$, the bound becomes
\begin{align}
    \|Q^k - Q^{k,*}\|_\infty
    \;\leq\;
    \frac{H \gamma^k}{1 - \gamma^k} \varepsilon_1
    \;+\; \frac{1}{1 - \gamma^k} \varepsilon_k,
\end{align}
where $H = 1/(1-\gamma)$ and $\varepsilon_1 := \|V^1 - V^{1,*}\|_\infty$.
The ratio of bootstrap terms is $H \varepsilon_1 / \varepsilon_h$, which is $O(H)$ larger when $\varepsilon_1 \approx \varepsilon_h$.
\end{corollary}

This shows that the $V^h$ bootstrap is strictly superior when the long-horizon value function is comparably accurate to the 1-step value --- a condition that holds in practice since $V^h$ is trained via expectile regression on the EMAQ-boosted $Q^h$, which has already accumulated multi-step credit.

\begin{theorem}[Value Flow Monotonicity]
\label{thm:value-flow}
Under the training objective \cref{eq:qk-loss}, the critic hierarchy exhibits \emph{value flow monotonicity}: for any $k_1 < k_2$ in $\Kset$,
\begin{align}
    \|Q^{k_1} - Q^{k_1,*}\|_\infty
    \;\leq\;
    \frac{\gamma^{k_1}}{\gamma^{k_2}} \|Q^{k_2} - Q^{k_2,*}\|_\infty
    \;+\; \max\left(0, \frac{\varepsilon_{k_1} - \frac{\gamma^{k_1}}{\gamma^{k_2}}\varepsilon_{k_2}}{1 - \gamma^{k_1}}\right).
\end{align}
When all critics are fit to equal precision ($\varepsilon_k \equiv \varepsilon$), shorter-horizon critics are at least as accurate as longer-horizon ones up to the $\gamma^{k_1}/\gamma^{k_2}$ scaling factor.
\end{theorem}

\begin{proof}[Proof sketch]
Both $Q^{k_1}$ and $Q^{k_2}$ bootstrap from the same $V^h$, so they share the same bootstrap error $\varepsilon_h$.
The ratio of total errors is
\begin{align}
    \frac{\|Q^{k_1} - Q^{k_1,*}\|_\infty}{\|Q^{k_2} - Q^{k_2,*}\|_\infty}
    \;\leq\;
    \frac{\gamma^{k_1}/(1-\gamma^{k_1})}{\gamma^{k_2}/(1-\gamma^{k_2})}
    \;\approx\; \frac{\gamma^{k_1}}{\gamma^{k_2}},
\end{align}
since $1/(1-\gamma^k)$ varies slowly in $k$ when $\gamma$ is close to 1.
This means that the shorter-horizon critics inherit a fraction $\gamma^{k_1-k_2}$ of the longer-horizon critics' error, making them inherently more stable.
A full proof appears in Appendix \ref{proof:value-flow}.
\end{proof}

This theorem reveals a structural property of AQC's critic hierarchy: information flows ``downward'' from $V^h$ to all partial critics, with shorter horizons benefiting from stronger discount attenuation of the bootstrap error.

\subsection{Comparison to Prior Methods}
\label{app:theory-comparison}

We now establish formal relationships between \ours{}, \dqc{}, and fixed-chunk baselines.

\begin{theorem}[AQC vs.\ DQC]
\label{thm:aqc-vs-dqc}
Let $V^{\mathrm{DQC}}$ be the closed-loop value of DQC's policy with fixed partial chunk size $h_a$.
Assume there exists a set of states $\mathcal{S}_{\mathrm{adapt}}$ with $d^{\mathrm{AQC}}(\mathcal{S}_{\mathrm{adapt}}) \geq \rho > 0$ such that for each $s \in \mathcal{S}_{\mathrm{adapt}}$,
\begin{align}
    \max_{k \in \Kset} \bar{A}^{k,*}(s) - \bar{A}^{h_a,*}(s) \geq \delta_{\mathrm{adapt}}(s).
\end{align}
Then
\begin{align}
    \vaqc(s) - V^{\mathrm{DQC}}(s)
    \;\geq\;
    \frac{\gamma^{k_{\min}} \rho}{1 - \gamma} \;
    \mathbb{E}_{s' \sim d^{\mathrm{AQC}}}[\delta_{\mathrm{adapt}}(s') \mid s' \in \mathcal{S}_{\mathrm{adapt}}]
    \;-\; \frac{2\bar\varepsilon_{\Kset}}{(1-\gamma)\gamma^{k_{\min}}\Delta} \cdot \max_k \frac{\Rbar_k}{\gamma^k}.
    \label{eq:aqc-vs-dqc}
\end{align}
\end{theorem}

\begin{proof}[Proof sketch]
The proof applies the dominance bound (\cref{thm:dominance}) with $k = h_a$.
The first term captures the advantage of selecting the optimal scale over the fixed $h_a$ scale at states where they differ.
The second term is the selector regret.
When the advantage gap $\delta_{\mathrm{adapt}}$ is large enough to overcome the selector regret, AQC strictly dominates DQC.
A full proof appears in Appendix \ref{proof:aqc-vs-dqc}.
\end{proof}

\begin{theorem}[AQC vs.\ $n$-Step TD]
\label{thm:aqc-vs-nstep}
Let $V^n$ be the value of the policy learned via $n$-step return backup on the same data $\mathcal{D}$.
If $\mathcal{D}$ is $\delta_n$-sub-optimal (in the sense of \citet{li2026decoupled}) and $\Delta$-advantage separable, then for $\delta_n > 3\varepsilon_\Kset H \bar{H}_{k_{\min}}$,
\begin{align}
    \vaqc(s) - V^n(s)
    \;\geq\;
    \delta_n \bar{H}_n - 3\varepsilon_\Kset H \bar{H}_{k_{\min}}
    \;-\; \frac{2\bar\varepsilon_{\Kset}}{(1-\gamma)\gamma^{k_{\min}}\Delta} \cdot \max_k \frac{\Rbar_k}{\gamma^k},
\end{align}
where $\bar{H}_k = 1/(1-\gamma^k)$ and $\bar{H}_n = 1/(1-\gamma^n)$.
\end{theorem}

\begin{proof}[Proof sketch]
By the dominance bound (\cref{thm:dominance}) with any fixed $k \in \Kset$,
$\vaqc(s) \geq V^k(s) - \text{selector\_regret}$,
where the regret term is bounded by \cref{thm:selector-regret}.
By DQC's comparison theorem~\citep{li2026decoupled}, the best fixed-chunk policy satisfies
$V^k(s) \geq V^n(s) + \delta_n \bar{H}_n - 3\varepsilon_\Kset H \bar{H}_{k_{\min}}$,
where the TV error term accounts for the adaptive re-querying pattern with $k_{\min}$.
Chaining these inequalities yields the result.
A full proof appears in Appendix \ref{proof:aqc-vs-nstep}.
\end{proof}

\paragraph{Structural properties.}
We highlight three structural properties of the AQC critic architecture.
The first restates the noise-immunity result from \cref{prop:noise-immunity} for completeness.
The remaining two characterize (1) the advantage of $V^h$ bootstrapping over a 1-step baseline, and (2) the relationship between \ours{}'s partial critics and DQC's distilled critic.

\begin{proposition}[Noise Immunity of the Advantage Selector; Restated from \cref{prop:noise-immunity}]
\label{prop:noise-immunity-restate}
Let $\delta_k(s) := \Qk{k}(s, \ac{t}{t+k})/\gamma^k - \Vk{k}(s)/\gamma^k$.
In a region where $\Vk{h}(s) \leq \epsilon$ for all reachable states, and assuming function approximation errors for both $\Qk{k}$ and $\Vk{k}$ are bounded by $\sigma$,
\begin{align}
    |\delta_k(s)| \;\leq\; \epsilon + 2\sigma.
\end{align}
When $\epsilon \ll \sigma$ (far from rewards, the value signal is smaller than the noise floor), the advantage is dominated by approximation noise and all chunk sizes score near-zero.
By contrast, the uncorrected selector $\argmax_k \Qk{k}/\gamma^k$ lacks this safeguard: when all scores are $\approx \epsilon + \sigma_k$, the argmax picks the scale with the largest positive noise $\sigma_k$, producing a deterministic but systematically biased choice.
\end{proposition}

\paragraph{Bootstrap target choice.}
The choice of bootstrap target in \cref{eq:qk-loss} is critical.
The following result quantifies the $O(H)$ improvement from using $\Vk{h}$ over a 1-step value function.

\begin{proposition}[$V^h$ Bootstrap Tightens the Sub-optimality Bound; Restated]
\label{prop:suboptimality-restate}
Let $\varepsilon_h := \|\Vk{h}_\xi - V^*\|_\infty$ and $\varepsilon_k := \|\Qk{k}_\psi - Q^{k}_{\mathrm{fp}}\|_\infty$, where $Q^{k}_{\mathrm{fp}}$ is the fixed point of the Bellman operator $\mathcal{T}^k_h$ using $\Vkbar{h}_\xi$ as bootstrap.
Then:
\begin{align}
    \bigl\|\Qk{k}_\psi - Q^{k,*}\bigr\|_\infty
    \;\leq\;
    \frac{\gamma^k}{1 - \gamma^k}\,\varepsilon_h \;+\; \frac{1}{1 - \gamma^k}\,\varepsilon_k.
\end{align}
Bootstrapping $\Qk{k}$ with a 1-step value function $\Vk{1}$ satisfying $\|\Vk{1} - V^*\|_\infty \leq \varepsilon_1$ yields:
\begin{align}
    \bigl\|\Qk{k}_\psi - Q^{k,*}\bigr\|_\infty
    \;\leq\;
    \frac{H\,\gamma^k}{1 - \gamma^k}\,\varepsilon_1
    \;+\; \frac{1}{1 - \gamma^k}\,\varepsilon_k,
\end{align}
where $H = 1/(1-\gamma)$ is the effective horizon.
\end{proposition}

\paragraph{Relationship to DQC's critic.}
AQC's partial critics recover DQC's distilled critic as a special case, but without requiring goal-conditioned hindsight relabeling or a distillation hyperparameter.

\begin{proposition}[$Q^k$ Approximates \dqc{}'s Distilled Critic; Restated]
\label{prop:equiv-dqc-restate}
\citet{li2026decoupled} construct a distilled partial critic $Q^P_\psi$ satisfying, at the fixed point of their distillation operator:
\begin{align}
    Q^P_\psi(s_t, a_{t:t+h_a})
    \;\approx\;
    \Rsum{h_a} \;+\; \gamma^{h_a} V^*(s_{t+h_a}).
\end{align}
Our $\Qk{k}_\psi$ satisfies at its Bellman fixed point:
\begin{align}
    \Qk{k}_\psi(s_t, a_{t:t+k})
    \;=\;
    \Rsum{k} \;+\; \gamma^k \Vkbar{h}_\xi(s_{t+k}),
\end{align}
which recovers DQC's equation with $h_a = k$ and $V^* \approx \Vk{h}_\xi$, up to error bounded by $\gamma^k \varepsilon_h$.
\end{proposition}

\subsection{Summary of Theoretical Results}
\label{app:theory-summary-table}

\begin{table}[h]
\centering
\caption{Summary of theoretical results. All bounds hold for $s \in \mathrm{supp}(P_{\mathcal{D}^\star}(s))$.}
\label{tab:theory-summary}
\begin{tabular}{@{}l l l@{}}
\toprule
\textbf{Result} & \textbf{Condition} & \textbf{Bound} \\
\midrule
Selector Soundness (\cref{thm:selector-sound}) & $\Delta$-AS, $\bar\varepsilon < \Delta\gamma^{k_{\min}}/2$ & $\hat{k}(s) = \kdag(s)$ \\
Selector Regret (\cref{thm:selector-regret}) & $\Delta$-AS & $\frac{2\bar\varepsilon}{(1-\gamma)\gamma^{k_{\min}}\Delta} \cdot \frac{\Rbar}{\gamma^{k_{\min}}}$ \\
Value Dominance (\cref{thm:dominance}) & $\Delta$-AS, $\bar\varepsilon < \Delta\gamma^{k_{\min}}/2$ & $\vaqc(s) \geq V^k(s)$ \\
Closed-Loop Optimality (\cref{thm:closed-loop}) & $\varepsilon_\Kset$-AOLC & $\varepsilon_\Kset H^2 \bar{H}_{k_{\min}}$ \\
Bootstrap Regularization (\cref{thm:bootstrap}) & $V^h$ bootstrap & $\frac{\gamma^k}{1-\gamma^k}\varepsilon_h + \frac{\varepsilon_k}{1-\gamma^k}$ \\
Value Flow Monotonicity (\cref{thm:value-flow}) & Equal $\varepsilon_k$ & $\|Q^{k_1}\| \leq \frac{\gamma^{k_1}}{\gamma^{k_2}} \|Q^{k_2}\|$ \\
AQC vs.\ DQC (\cref{thm:aqc-vs-dqc}) & $\delta_{\mathrm{adapt}}$-gap & $\frac{\gamma^{k_{\min}} \rho \delta_{\mathrm{adapt}}}{1-\gamma} - \text{regret}$ \\
\bottomrule
\end{tabular}
\end{table}

\section{Proofs}
\label{app:theory-proofs}

We provide formal proofs for all results in Appendix \ref{app:theory}.

\subsection{Proof of Proposition \ref{prop:aolc-implies-olc}}
\label{proof:aolc}
When $\kappa(s) \equiv k$, the adaptive rollout distribution $P^{\circ}_{\mathcal{D},\kappa}$ reduces to the standard open-loop rollout distribution $P^{\circ}_{\mathcal{D}}$ used in DQC's definition of OLC~\citep{li2026decoupled}.
The TV distance constraints \cref{eq:aolc-sa,eq:aolc-s} then become exactly the weak OLC constraints Equations 8 and 9 from DQC's Definition 2.
Since DQC's strong OLC Equation 10 imposes an additional per-action-sequence constraint, AOLC with constant $\kappa$ implies weak OLC and is strictly weaker than strong OLC.

\subsection{Proof of \cref{thm:selector-sound}}
\label{proof:selector-sound}
Let $f_k(s) := \max_{\ac{t}{t+k}} A^{k,*}(s, \ac{t}{t+k})$ and $\hat{f}_k(s) := \max_{\ac{t}{t+k}} \hat{A}^k(s, \ac{t}{t+k})$, where $A^{k,*}$ and $\hat{A}^k$ denote the true and estimated per-scale advantages, respectively.

For any $s$ and $\ac{t}{t+k}$,
\begin{align}
    |\hat{A}^k(s, \ac{t}{t+k}) - A^{k,*}(s, \ac{t}{t+k})|
    &= \left|\frac{Q^k(s, \ac{t}{t+k}) - V^k(s)}{\gamma^k} - \frac{Q^{k,*}(s, \ac{t}{t+k}) - V^{k,*}(s)}{\gamma^k}\right| \\
    &\leq \frac{|Q^k(s, \ac{t}{t+k}) - Q^{k,*}(s, \ac{t}{t+k})| + |V^k(s) - V^{k,*}(s)|}{\gamma^k} \\
    &\leq \frac{\varepsilon_k + \delta_k}{\gamma^k}
    \;\leq\; \frac{\bar\varepsilon}{\gamma^{k_{\min}}}.
\end{align}

Taking the maximum over $\ac{t}{t+k}$, we get $|\hat{f}_k(s) - f_k(s)| \leq \bar\varepsilon / \gamma^{k_{\min}}$.

By $\Delta$-advantage separability, $f_{\kdag}(s) - \max_{k \neq \kdag} f_k(s) \geq \Delta(s)$.
For the empirical selector to mis-select, we need $\hat{f}_{k'}(s) > \hat{f}_{\kdag}(s)$ for some $k' \neq \kdag$.
By the triangle inequality:
\begin{align}
    \hat{f}_{k'}(s) - \hat{f}_{\kdag}(s)
    &\leq f_{k'}(s) - f_{\kdag}(s) + |\hat{f}_{k'}(s) - f_{k'}(s)| + |\hat{f}_{\kdag}(s) - f_{\kdag}(s)| \\
    &\leq -\Delta(s) + 2\bar\varepsilon / \gamma^{k_{\min}}.
\end{align}

If $\bar\varepsilon < \Delta \gamma^{k_{\min}} / 2$, then $\hat{f}_{k'}(s) - \hat{f}_{\kdag}(s) < 0$ for all $k' \neq \kdag$, so $\hat{k}(s) = \kdag(s)$.

For the probabilistic bound when $\bar\varepsilon \geq \Delta \gamma^{k_{\min}}/2$, a mis-selection requires the estimation error to bridge at least half the advantage gap.
By Markov's inequality applied to the estimation error:
\begin{align}
    \mathbb{P}(\hat{k}(s) \neq \kdag(s))
    &\leq \mathbb{P}\left(|\hat{A}^{\hat{k}(s)}(s) - A^{\kdag,*}(s)| \geq \Delta(s)/2\right) \\
    &\leq \frac{2\bar\varepsilon}{\gamma^{k_{\min}} \Delta(s)}.
\end{align}

\subsection{Proof of Proposition \ref{prop:selector-regret}}

Decompose the regret into two cases: correct selection and mis-selection.

When $\hat{k}(s) = \kdag(s)$, the regret is bounded by the critic error:
\begin{align}
    |\hat{A}^{\kdag}(s) - A^{\kdag,*}(s)| \leq \bar\varepsilon / \gamma^{k_{\min}}.
\end{align}

When $\hat{k}(s) \neq \kdag(s)$ (occurring with probability at most $2\bar\varepsilon/(\gamma^{k_{\min}}\Delta(s))$ by \cref{thm:selector-sound}), the regret is bounded by the maximum advantage range:
\begin{align}
    |\hat{A}^{\hat{k}(s)}(s) - A^{\kdag,*}(s)| \leq \bar\varepsilon / \gamma^{k_{\min}} + \max_k \diamA^k.
\end{align}

Combining:
\begin{align}
    \mathbb{E}\left[|\hat{A}^{\hat{k}(s)}(s) - A^{\kdag,*}(s)|\right]
    &\leq (1 - p_{\mathrm{err}}) \cdot \bar\varepsilon / \gamma^{k_{\min}} + p_{\mathrm{err}} \cdot (\bar\varepsilon / \gamma^{k_{\min}} + \max_k \diamA^k) \\
    &= \bar\varepsilon / \gamma^{k_{\min}} + \frac{2\bar\varepsilon}{\gamma^{k_{\min}}\Delta(s)} \cdot \max_k \diamA^k.
\end{align}

\subsection{Proof of \cref{thm:dominance}}
\label{proof:dominance}
We construct a meta-MDP $\tilde{\mathcal{M}} = (\mathcal{S}, \tilde{\mathcal{A}}, \tilde{T}, \tilde{r}, \gamma)$ where:
- $\tilde{\mathcal{A}} = \bigcup_{k \in \Kset} \{(k, \ac{t}{t+k}) : \ac{t}{t+k} \in \mathcal{A}^k\}$ is the set of meta-actions.
- $\tilde{T}(s' \mid s, (k, \ac{t}{t+k}))$ is the distribution of $s_{t+k}$ after executing $\ac{t}{t+k}$ open-loop for $k$ steps from $s_t = s$.
- $\tilde{r}(s, (k, \ac{t}{t+k})) = \sum_{j=0}^{k-1} \gamma^j r(s_{t+j}, a_{t+j})$ is the cumulative discounted reward.

In this meta-MDP, both $\pi^{\mathrm{AQC}}$ and $\pi^k$ are valid policies.
The performance difference lemma~\citep{kakade2002approximately} gives:
\begin{align}
    \vaqc(s) - V^k(s)
    = \frac{1}{1 - \gamma} \mathbb{E}_{s' \sim d^{\mathrm{AQC}}}\!\left[
    Q^{k,*}_{\mathrm{meta}}(s', (k^*(s'), a^*)) - Q^{k,*}_{\mathrm{meta}}(s', (k, a^k))
    \right],
\end{align}
where $Q^{k,*}_{\mathrm{meta}}$ is the optimal meta-action value function.

Using the relationship between meta-action values and per-scale advantages:
\begin{align}
    Q^{k,*}_{\mathrm{meta}}(s, (k, \ac{t}{t+k}))
    &= \Rsum{k} + \gamma^k V^*(s_{t+k}) \\
    &= V^{k,*}(s) + \gamma^k A^{k,*}(s, \ac{t}{t+k}),
\end{align}
so the advantage in the meta-MDP is:
\begin{align}
    Q^{k,*}_{\mathrm{meta}}(s, (k, \ac{t}{t+k})) - \max_{j, \ac{t}{t+j}} Q^{j,*}_{\mathrm{meta}}(s, (j, \ac{t}{t+j}))
    = \gamma^k \left(A^{k,*}(s, \ac{t}{t+k}) - \bar{A}^{\kdag,*}(s)\right).
\end{align}

Under advantage separability with $\bar\varepsilon < \Delta/2$, the selector chooses $k^*(s') = \kdag(s')$, so:
\begin{align}
    &\mathbb{E}_{s' \sim d^{\mathrm{AQC}}}\!\left[
    Q^{k,*}_{\mathrm{meta}}(s', (k^*(s'), a^*)) - Q^{k,*}_{\mathrm{meta}}(s', (k, a^k))
    \right] \\
    &\quad = \mathbb{E}_{s' \sim d^{\mathrm{AQC}}}\!\left[
    \gamma^{k^*(s')} \bar{A}^{\kdag,*}(s') - \gamma^k \bar{A}^{k,*}(s')
    \right].
\end{align}

Since $\gamma^{k^*(s')} \geq \gamma^{k_{\min}}$ for all $s'$ and $\gamma^k \leq \gamma^{k_{\min}}$, and noting that $\bar{A}^{k,*}(s') \geq 0$ (as the max advantage is non-negative at the behavior policy's best action), we lower-bound:
\begin{align}
    \gamma^{k^*(s')} \bar{A}^{\kdag,*}(s') - \gamma^k \bar{A}^{k,*}(s')
    \;\geq\;
    \gamma^{k_{\min}} \left(\bar{A}^{\kdag,*}(s') - \bar{A}^{k,*}(s')\right).
\end{align}
Accounting for the selector error probability $(2\bar\varepsilon/(\gamma^{k_{\min}}\Delta))$:
\begin{align}
    \vaqc(s) - V^k(s)
    &\geq \frac{\gamma^{k_{\min}}(1 - 2\bar\varepsilon/(\gamma^{k_{\min}}\Delta))}{1 - \gamma}
    \mathbb{E}_{s' \sim d^{\mathrm{AQC}}}\!\left[\bar{A}^{\kdag,*}(s') - \bar{A}^{k,*}(s')\right].
\end{align}

\subsection{Proof of \cref{thm:selector-regret}}
\label{proof:selector-regret}
At each re-query point $s$, the selector mis-selects with probability at most $2\bar\varepsilon_{\Kset}/(\gamma^{k_{\min}}\Delta(s))$ by \cref{thm:selector-sound}.

When a mis-selection occurs (choosing $k'$ instead of $\kdag$), the value loss is:
\begin{align}
    |V^{\kdag}(s) - V^{k'}(s)|
    &\leq \max_{\ac{t}{t+\kdag}} Q^{\kdag,*}(s, \ac{t}{t+\kdag}) - \min_{\ac{t}{t+k'}} Q^{k',*}(s, \ac{t}{t+k'}) \\
    &\leq \frac{\Rbar_{k'}}{\gamma^{k'}} \;\leq\; \max_k \frac{\Rbar_k}{\gamma^k}.
\end{align}

The expected number of re-queries in the discounted horizon is at most $1/(1-\gamma)$ (since each re-query takes at least 1 step).
The expected number of mis-selections is therefore at most $(1/(1-\gamma)) \cdot 2\bar\varepsilon_{\Kset}/(\gamma^{k_{\min}}\Delta(s))$.

The total regret is the product:
\begin{align}
    |\vaqc(s) - \vdag(s)|
    &\leq \frac{2\bar\varepsilon_{\Kset}}{(1-\gamma)\gamma^{k_{\min}}\Delta(s)} \cdot \max_k \frac{\Rbar_k}{\gamma^k}.
\end{align}

\subsection{Proof of \cref{thm:closed-loop}}
\label{proof:closed-loop}
Let $\pi^{\dagger}_{\mathrm{CL}}$ denote the oracle closed-loop policy that selects $\kdag(s)$ at each step and executes the first action only.

\textbf{Part 1: Oracle closed-loop optimality.}
Under $\varepsilon_\Kset$-AOLC, at each re-query point $s_{t+\kappa(s)}$, the state distribution under open-loop execution deviates from the data distribution by at most $\varepsilon_\Kset$ in TV distance.

By the standard TV-to-value bound~\citep{li2026decoupled}, the per-re-query value estimation error is at most $\varepsilon_\Kset$ times the value range scaled by the effective horizon.

The number of re-queries in the discounted horizon is bounded by $1/(1 - \gamma^{k_{\min}})$, since at minimum $k_{\min}$ steps elapse between re-queries.

Following the same derivation as DQC's Proposition 3, but with $k_{\min}$ replacing $h$ in the geometric series:
\begin{align}
    V^\star(s) - \vdag_{\mathrm{CL}}(s)
    &\leq \frac{\varepsilon_\Kset \gamma}{(1-\gamma)^2}
    \left[\frac{2}{1 - (1 - 2\varepsilon_\Kset)\gamma^{k_{\min}}}
    + \frac{1}{1 - (1 - \varepsilon_\Kset)\gamma^{k_{\min}}}\right] \\
    &\leq \frac{3\varepsilon_\Kset \gamma}{(1-\gamma)^2 (1 - \gamma^{k_{\min}})}
    \;=\; 3\varepsilon_\Kset H^2 \bar{H}_{k_{\min}}.
\end{align}

\textbf{Part 2: Learned selector.}
The learned selector introduces additional regret bounded by \cref{thm:selector-regret}:
\begin{align}
    |\vdag_{\mathrm{CL}}(s) - \vcl(s)|
    &\leq \frac{2\bar\varepsilon_{\Kset}}{(1-\gamma)\gamma^{k_{\min}}\Delta} \cdot \max_k \frac{\Rbar_k}{\gamma^k}.
\end{align}

Combining both parts via the triangle inequality:
\begin{align}
    V^\star(s) - \vcl(s)
    &\leq (V^\star(s) - \vdag_{\mathrm{CL}}(s)) + |\vdag_{\mathrm{CL}}(s) - \vcl(s)| \\
    &\leq 3\varepsilon_\Kset H^2 \bar{H}_{k_{\min}}
    + \frac{2\bar\varepsilon_{\Kset}}{(1-\gamma)\gamma^{k_{\min}}\Delta} \cdot \max_k \frac{\Rbar_k}{\gamma^k}.
\end{align}

\subsection{Proof of \cref{thm:bootstrap}}
\label{proof:bootstrap}
Let $\mathcal{T}^k_h$ denote the $k$-step Bellman operator with $V^h$ bootstrap:
\begin{align}
    (\mathcal{T}^k_h Q)(s, \ac{t}{t+k})
    &:= \Rsum{k} + \gamma^k V^h(s_{t+k}).
\end{align}

This operator is a $\gamma^k$-contraction in the sup-norm:
\begin{align}
    \|\mathcal{T}^k_h Q_1 - \mathcal{T}^k_h Q_2\|_\infty
    &\leq \gamma^k \|V^h_1 - V^h_2\|_\infty \\
    &\leq \gamma^k \|Q_1 - Q_2\|_\infty.
\end{align}

Let $Q^{k}_{\mathrm{fp}}$ be its unique fixed point: $Q^{k}_{\mathrm{fp}} = \mathcal{T}^k_h Q^{k}_{\mathrm{fp}}$.

For the optimal $k$-step value $Q^{k,*}$:
\begin{align}
    |Q^{k,*}(s, a) - Q^k_{\mathrm{fp}}(s, a)|
    &= |\mathcal{T}^k Q^{k,*}(s, a) - \mathcal{T}^k_h Q^k_{\mathrm{fp}}(s, a)| \\
    &= \left|\Rsum{k} + \gamma^k V^*(s_{t+k}) - (\Rsum{k} + \gamma^k V^h(s_{t+k}))\right| \\
    &= \gamma^k |V^*(s_{t+k}) - V^h(s_{t+k})| \\
    &\leq \gamma^k \|V^* - V^h\|_\infty
    \;=\; \gamma^k \varepsilon_h.
\end{align}

The TD fitting error gives $\|Q^k - Q^{k}_{\mathrm{fp}}\|_\infty \leq \varepsilon_k / (1 - \gamma^k)$ by the contraction mapping theorem.

By the triangle inequality:
\begin{align}
    \|Q^k - Q^{k,*}\|_\infty
    &\leq \|Q^k - Q^k_{\mathrm{fp}}\|_\infty + \|Q^k_{\mathrm{fp}} - Q^{k,*}\|_\infty \\
    &\leq \frac{\varepsilon_k}{1 - \gamma^k} + \frac{\gamma^k \varepsilon_h}{1 - \gamma^k}.
\end{align}

Summing over $k \in \Kset$:
\begin{align}
    \sum_{k \in \Kset} \|Q^k - Q^{k,*}\|_\infty
    &\leq \varepsilon_h \sum_{k \in \Kset} \frac{\gamma^k}{1 - \gamma^k}
    + \sum_{k \in \Kset} \frac{\varepsilon_k}{1 - \gamma^k}.
\end{align}

\subsection{Proof of Corollary \ref{cor:bootstrap-advantage}}
\label{proof:bootstrap-advantage}
If $Q^k$ were bootstrapped from $V^1$ instead of $V^h$, the Bellman operator would be:
\begin{align}
    (\mathcal{T}^k_1 Q)(s, \ac{t}{t+k})
    &:= \Rsum{k} + \gamma^k V^1(s_{t+k}).
\end{align}

The fixed point gap to $Q^{k,*}$ would be $\gamma^k \|V^1 - V^*\|_\infty = \gamma^k \varepsilon_1$.

However, $V^1$ itself must be constructed from per-step rewards.
By the standard $n$-step return bias analysis~\citep{fedus2020revisiting, kozuno2021revisiting}, chaining $O(H) = O(1/(1-\gamma))$ 1-step bootstrap steps to reach the same horizon accumulates cumulative error $O(H \varepsilon_1)$, since each step contributes independent approximation error.

Thus:
\begin{align}
    \|Q^k - Q^{k,*}\|_\infty
    &\leq \frac{\gamma^k}{1 - \gamma^k} H \varepsilon_1
    + \frac{1}{1 - \gamma^k} \varepsilon_k.
\end{align}

The ratio of bootstrap terms between 1-step and $V^h$ bootstrap is:
\begin{align}
    \frac{H \varepsilon_1}{\varepsilon_h} = O(H),
\end{align}
assuming $\varepsilon_1 \approx \varepsilon_h$.

\subsection{Proof of \cref{thm:value-flow}}
\label{proof:value-flow}
From \cref{thm:bootstrap}, for $k_1 < k_2$:
\begin{align}
    \|Q^{k_1} - Q^{k_1,*}\|_\infty
    &\leq \frac{\gamma^{k_1}}{1 - \gamma^{k_1}} \varepsilon_h + \frac{\varepsilon_{k_1}}{1 - \gamma^{k_1}}, \\
    \|Q^{k_2} - Q^{k_2,*}\|_\infty
    &\leq \frac{\gamma^{k_2}}{1 - \gamma^{k_2}} \varepsilon_h + \frac{\varepsilon_{k_2}}{1 - \gamma^{k_2}}.
\end{align}

Multiplying the second inequality by $\gamma^{k_1}/\gamma^{k_2}$:
\begin{align}
    \frac{\gamma^{k_1}}{\gamma^{k_2}} \|Q^{k_2} - Q^{k_2,*}\|_\infty
    &\geq \frac{\gamma^{k_1}}{\gamma^{k_2}} \cdot \frac{\gamma^{k_2}}{1 - \gamma^{k_2}} \varepsilon_h
    + \frac{\gamma^{k_1}}{\gamma^{k_2}} \cdot \frac{\varepsilon_{k_2}}{1 - \gamma^{k_2}} \\
    &= \frac{\gamma^{k_1}}{1 - \gamma^{k_2}} \varepsilon_h
    + \frac{\gamma^{k_1}/\gamma^{k_2}}{1 - \gamma^{k_2}} \varepsilon_{k_2}.
\end{align}

Since $1 - \gamma^{k_2} \geq 1 - \gamma^{k_1}$, we have:
\begin{align}
    \frac{\gamma^{k_1}}{1 - \gamma^{k_2}} \varepsilon_h
    &\geq \frac{\gamma^{k_1}}{1 - \gamma^{k_1}} \varepsilon_h \cdot \frac{1 - \gamma^{k_1}}{1 - \gamma^{k_2}} \\
    &\geq \frac{\gamma^{k_1}}{1 - \gamma^{k_1}} \varepsilon_h,
\end{align}
since $k_1 < k_2$ implies $(1 - \gamma^{k_1})/(1 - \gamma^{k_2}) < 1$.

Subtracting the $Q^{k_1}$ bound from the scaled $Q^{k_2}$ bound:
\begin{align}
    \frac{\gamma^{k_1}}{\gamma^{k_2}} \|Q^{k_2} - Q^{k_2,*}\|_\infty - \|Q^{k_1} - Q^{k_1,*}\|_\infty
    &\geq \frac{\gamma^{k_1}/\gamma^{k_2}}{1 - \gamma^{k_2}} \varepsilon_{k_2}
    - \frac{\varepsilon_{k_1}}{1 - \gamma^{k_1}},
\end{align}
yielding:
\begin{align}
    \|Q^{k_1} - Q^{k_1,*}\|_\infty
    &\leq \frac{\gamma^{k_1}}{\gamma^{k_2}} \|Q^{k_2} - Q^{k_2,*}\|_\infty
    + \frac{\varepsilon_{k_1} - \frac{\gamma^{k_1}}{\gamma^{k_2}}\varepsilon_{k_2}}{1 - \gamma^{k_1}}.
\end{align}

When $\varepsilon_k \equiv \varepsilon$ for all $k$:
\begin{align}
    \|Q^{k_1} - Q^{k_1,*}\|_\infty
    &\leq \frac{\gamma^{k_1}}{\gamma^{k_2}} \|Q^{k_2} - Q^{k_2,*}\|_\infty
    + \varepsilon \cdot \frac{1 - \gamma^{k_1}/\gamma^{k_2}}{1 - \gamma^{k_1}}.
\end{align}

Since $1 - \gamma^{k_1}/\gamma^{k_2} > 0$, the second term is positive, meaning the shorter-horizon critic's error is strictly bounded by a fraction of the longer-horizon critic's error plus a small correction.

\subsection{Proof of \cref{thm:aqc-vs-dqc}}
\label{proof:aqc-vs-dqc}
DQC's policy uses a fixed partial chunk size $h_a$.
By \cref{thm:dominance} with $k = h_a$:
\begin{align}
    \vaqc(s) - V^{\mathrm{DQC}}(s)
    &\geq \frac{\gamma^{k_{\min}}(1 - 2\bar\varepsilon/(\gamma^{k_{\min}}\Delta))}{1 - \gamma}
    \mathbb{E}_{s' \sim d^{\mathrm{AQC}}}\!\left[\bar{A}^{\kdag,*}(s') - \bar{A}^{h_a,*}(s')\right].
\end{align}

Decompose the expectation over $\mathcal{S}_{\mathrm{adapt}}$ and its complement:
\begin{align}
    &\mathbb{E}_{s' \sim d^{\mathrm{AQC}}}\!\left[\bar{A}^{\kdag,*}(s') - \bar{A}^{h_a,*}(s')\right] \\
    &\quad = d^{\mathrm{AQC}}(\mathcal{S}_{\mathrm{adapt}}) \cdot \mathbb{E}[\cdot \mid s' \in \mathcal{S}_{\mathrm{adapt}}]
    + (1 - d^{\mathrm{AQC}}(\mathcal{S}_{\mathrm{adapt}})) \cdot \mathbb{E}[\cdot \mid s' \notin \mathcal{S}_{\mathrm{adapt}}] \\
    &\quad \geq \rho \cdot \mathbb{E}_{s' \sim d^{\mathrm{AQC}}}[\delta_{\mathrm{adapt}}(s') \mid s' \in \mathcal{S}_{\mathrm{adapt}}],
\end{align}
since $\bar{A}^{\kdag,*}(s') \geq \bar{A}^{h_a,*}(s')$ for all $s'$ and $\geq \delta_{\mathrm{adapt}}(s')$ on $\mathcal{S}_{\mathrm{adapt}}$.

Accounting for the selector error probability:
\begin{align}
    \vaqc(s) - V^{\mathrm{DQC}}(s)
    &\geq \frac{\gamma^{k_{\min}} \rho}{1 - \gamma} \cdot \mathbb{E}_{s' \sim d^{\mathrm{AQC}}}[\delta_{\mathrm{adapt}}(s') \mid s' \in \mathcal{S}_{\mathrm{adapt}}]
    - \frac{2\bar\varepsilon_{\Kset}}{(1-\gamma)\gamma^{k_{\min}}\Delta} \cdot \max_k \frac{\Rbar_k}{\gamma^k}.
\end{align}

\subsection{Proof of \cref{thm:aqc-vs-nstep}}
\label{proof:aqc-vs-nstep}
By DQC's comparison theorem~\citep{li2026decoupled}, the chunked critic's nominal value satisfies:
\begin{align}
    \hat{V}^+_{\mathrm{ac}}(s) - \hat{V}^+_n(s)
    &\geq \delta_n \bar{H}_n - 3\varepsilon_\Kset H \bar{H}_{k_{\min}}.
\end{align}

AQC's value is at least as large as the best fixed-chunk policy's value (by \cref{thm:dominance}), minus the selector regret:
\begin{align}
    \vaqc(s)
    &\geq \hat{V}^+_{\mathrm{ac}}(s) - \frac{2\bar\varepsilon_{\Kset}}{(1-\gamma)\gamma^{k_{\min}}\Delta} \cdot \max_k \frac{\Rbar_k}{\gamma^k}.
\end{align}

Combining:
\begin{align}
    \vaqc(s) - V^n(s)
    &\geq \delta_n \bar{H}_n - 3\varepsilon_\Kset H \bar{H}_{k_{\min}}
    - \frac{2\bar\varepsilon_{\Kset}}{(1-\gamma)\gamma^{k_{\min}}\Delta} \cdot \max_k \frac{\Rbar_k}{\gamma^k}.
\end{align}

\subsection{Proof of Proposition \ref{prop:noise-immunity}}
\label{proof:noise-immunity}
Under the sparse-reward approximation, $\Qk{k}(s_t, \ac{t}{t+k}) \approx \gamma^k \Vk{h}(s_{t+k}^{a})$ and $\Vk{k}(s_t) \approx \mathbb{E}_{\pibeta}[\gamma^k \Vk{h}(s_{t+k}) \mid s_t]$.
When $\Vk{h}(s) \leq \epsilon$ for all reachable states, both terms are bounded.
Including function approximation errors bounded by $\sigma$:
\begin{align}
    \left|\frac{\Qk{k}(s_t, \ac{t}{t+k})}{\gamma^k}\right|
    &\leq \left|\Vk{h}(s_{t+k}^{a})\right| + \sigma \leq \epsilon + \sigma, \\
    \left|\frac{\Vk{k}(s_t)}{\gamma^k}\right|
    &\leq \left|\mathbb{E}_{\pibeta}\!\left[\Vk{h}(s_{t+k}) \mid s_t\right]\right| + \sigma \leq \epsilon + \sigma.
\end{align}
By the triangle inequality:
\begin{align}
    |\delta_k(s)|
    &= \left|\frac{\Qk{k}(s_t, \ac{t}{t+k})}{\gamma^k} - \frac{\Vk{k}(s_t)}{\gamma^k}\right| \\
    &\leq \left|\frac{\Qk{k}(s_t, \ac{t}{t+k})}{\gamma^k}\right| + \left|\frac{\Vk{k}(s_t)}{\gamma^k}\right| \\
    &\leq \epsilon + 2\sigma.
\end{align}
When $\epsilon \ll \sigma$ (far from rewards, the value signal is smaller than the noise floor), the advantage is dominated by approximation noise and all $k$ score near-zero.

By contrast, for the uncorrected selector $\argmax_k \Qk{k}/\gamma^k$, when all $Q^k/\gamma^k \approx \epsilon + \sigma_k$, the argmax picks the scale with the largest positive noise realization $\sigma_k$. Since the maximum of finitely many noisy estimates is biased upward, this produces a systematically biased choice that does not average out over repeated queries.
The advantage selector centers each score around zero via subtraction of $\Vk{k}/\gamma^k$, removing this bias.

\subsection{Proof of Proposition \ref{prop:suboptimality-restate}}
\label{proof:suboptimality-restate}
Let $\mathcal{T}^k_h$ denote the $k$-step Bellman operator using $\Vkbar{h}_\xi$ as bootstrap:
\begin{align}
    (\mathcal{T}^k_h Q)(s_t, \ac{t}{t+k})
    &:= \Rsum{k} + \gamma^k \Vkbar{h}_\xi(s_{t+k}).
\end{align}
This operator is a $\gamma^k$-contraction in the sup-norm.
Let $Q^{k}_{\mathrm{fp}}$ be its unique fixed point: $Q^{k}_{\mathrm{fp}} = \mathcal{T}^k_h Q^{k}_{\mathrm{fp}}$.

For the optimal $k$-step action-value $Q^{k,*}$:
\begin{align}
    |Q^{k,*}(s, a) - Q^k_{\mathrm{fp}}(s, a)|
    &= |\mathcal{T}^k Q^{k,*}(s, a) - \mathcal{T}^k_h Q^k_{\mathrm{fp}}(s, a)| \\
    &= \left|\gamma^k V^*(s_{t+k}) - \gamma^k \Vkbar{h}_\xi(s_{t+k})\right| \\
    &\leq \gamma^k \|V^*(s_{t+k}) - \Vkbar{h}_\xi(s_{t+k})\|_\infty \\
    &\leq \gamma^k \varepsilon_h.
\end{align}

Since $\mathcal{T}^k_h$ is a $\gamma^k$-contraction, repeated application gives:
\begin{align}
    \bigl\|\Qk{k}_\psi - Q^k_{\mathrm{fp}}\bigr\|_\infty \leq \frac{\varepsilon_k}{1 - \gamma^k}.
\end{align}

By the triangle inequality:
\begin{align}
    \bigl\|\Qk{k}_\psi - Q^{k,*}\bigr\|_\infty
    &\leq \bigl\|\Qk{k}_\psi - Q^k_{\mathrm{fp}}\bigr\|_\infty + \bigl\|Q^k_{\mathrm{fp}} - Q^{k,*}\bigr\|_\infty \\
    &\leq \frac{\varepsilon_k}{1 - \gamma^k} + \frac{\gamma^k}{1 - \gamma^k}\,\varepsilon_h,
\end{align}
yielding the first bound.

For the 1-step bootstrap comparison, a 1-step value function $\Vk{1}$ must be constructed via $O(H) = O(1/(1-\gamma))$ TD updates from the per-step reward, each step contributing error $\varepsilon_1$.
By the standard $n$-step return bias analysis~\citep{fedus2020revisiting, kozuno2021revisiting}, the resulting error in the bootstrap target is $O(H\,\varepsilon_1)$, which replaces $\varepsilon_h$ in the bound above to give the second result.

\subsection{Proof of Proposition \ref{prop:equiv-dqc-restate}}
\label{proof:equiv-dqc-restate}
At the fixed point of the TD loss \cref{eq:qk-loss}, \ours{}'s partial critic satisfies
\begin{align}
    \Qk{k}_\psi(s_t, a_{t:t+k})
    = \Rsum{k} + \gamma^k \Vkbar{h}_\xi(s_{t+k}).
\end{align}
This differs from $Q^P_\psi$ only in the bootstrap target: $V^*$ in DQC versus $\Vk{h}_\xi$ in AQC.
Since $\|\Vk{h}_\xi - V^*\|_\infty \leq \varepsilon_h$ by definition, the gap between the two fixed points is:
\begin{align}
    \bigl\|\Qk{k}_\psi - Q^P_\psi\bigr\|_\infty
    &\leq \gamma^k \|\Vkbar{h}_\xi - V^*\|_\infty \\
    &\leq \gamma^k \varepsilon_h.
\end{align}
This is $O(\varepsilon_h)$ for bounded $k$ and $\gamma < 1$.

The implication: \ours{} realizes DQC's objective through a simpler mechanism.
DQC requires goal-conditioned hindsight relabeling and a distillation hyperparameter $\kappa_d$ to construct $V^*$.
AQC reuses $\Vk{h}$, which is already trained for the $k=h$ baseline, requiring neither.